\documentclass{article}

\usepackage[main, final]{neurips_2026}
\setcitestyle{round}

\usepackage[utf8]{inputenc} 
\usepackage[T1]{fontenc}    
\usepackage{hyperref}       
\hypersetup{
    colorlinks = true,
    linkcolor = red,
    citecolor = blue,
    urlcolor=blue
    }
\usepackage{url}            
\usepackage{booktabs}       
\usepackage{amsfonts}       
\usepackage{nicefrac}       
\usepackage{microtype}      
\usepackage{xcolor}         
\definecolor{scbmblue}{HTML}{4878CF}
\definecolor{scbmred}{HTML}{C44E52}

\usepackage{amsmath}
\usepackage{amssymb}
\usepackage{mathtools}
\usepackage{amsthm}

\usepackage{algorithm}
\usepackage{algpseudocode}

\usepackage{titletoc}
\usepackage{bbm}
\usepackage[normalem]{ulem}

\usepackage[capitalize,noabbrev]{cleveref}
\usepackage{wrapfig}

\usepackage{tikz} 
\usetikzlibrary{decorations.text}
\usetikzlibrary{arrows.meta, bending}
\usetikzlibrary{positioning}
\usetikzlibrary{bayesnet}
\usetikzlibrary{fit}
\usetikzlibrary{shapes,arrows}
\usetikzlibrary{shapes.multipart}
\usetikzlibrary{fit}

\def\*#1{\mathbf{#1}}

\theoremstyle{plain}
\newtheorem{theorem}{Theorem}[section]
\newtheorem{proposition}[theorem]{Proposition}
\newtheorem{lemma}[theorem]{Lemma}

\theoremstyle{definition}
\newtheorem{definition}[theorem]{Definition}

\theoremstyle{remark}
\newtheorem{remark}[theorem]{Remark}
\newenvironment{reptheorem}[1]{%
  \begin{trivlist}\item[\hskip\labelsep\bfseries Theorem~#1.]\itshape}{\end{trivlist}}
\newenvironment{repproposition}[1]{%
  \begin{trivlist}\item[\hskip\labelsep\bfseries Proposition~#1.]\itshape}{\end{trivlist}}
\newenvironment{replemma}[1]{%
  \begin{trivlist}\item[\hskip\labelsep\bfseries Lemma~#1.]\itshape}{\end{trivlist}}

\newcommand{\T}{\mathbf{T}}
\newcommand{\X}{\mathbf{X}}
\newcommand{\Y}{\mathbf{Y}}
\newcommand{\x}{\mathbf{x}}

\newcommand{\z}{\mathbf{z}}
\newcommand{\pa}{\mathrm{pa}}
\newcommand{\ch}{\mathrm{ch}}

\newcommand{\cZ}{\mathcal{Z}}

\newcommand{\cE}{\mathcal{E}}
\newcommand{\cF}{\mathcal{F}}

\newcommand{\cD}{\mathcal{D}}

\newcommand{\cH}{\mathcal{H}}

\newcommand{\cG}{\mathcal{G}}

\newcommand{\Z}{\mathbb{Z}}

\newcommand{\cB}{\mathcal{B}}

\newcommand{\cV}{\mathcal{V}}

\newcommand{\cX}{\mathcal{X}}

\renewcommand{\Z}{\mathbf{Z}}

\title{Structural Causal Bottleneck Models}

\author{%
  Simon Bing\thanks{Equal contribution.} \\
  Technische Universität Berlin\\
  and \\
  Department of Computer Science\\
  University of Potsdam\\
  \texttt{bing@campus.tu-berlin.de} \\
  \And
  Jonas Wahl\footnotemark[1] \\
  German Research Center for\\
  Artificial Intelligence (DFKI) \\
 and \\
  Department of Philosophy\\
  University of Bergen\\
  \And
  Jakob Runge \\
  Department of Computer Science\\
  University of Potsdam\\
}

\begin{document}

\maketitle

\begin{abstract}
    We introduce structural causal bottleneck models (SCBMs), a novel class of structural causal models in which causal effects between high-dimensional variables are mediated by low-dimensional summary statistics, or \emph{bottlenecks}. SCBMs provide a flexible framework for mechanism-specific, target-dependent dimension reduction while remaining estimable via standard learning algorithms. We prove that bottleneck variables are identifiable up to bijection from observational data, and validate this experimentally. We then study causal effect estimation in linear SCBMs under three standard adjustment strategies. For instrumental variables, we identify a population-level non-identifiability of naïve two-stage least squares when all variables are high-dimensional, and show that bottleneck representations restore identification. For mediation and backdoor adjustment, we prove that conditioning on bottleneck variables yields valid effect estimates with explicit finite-sample efficiency gains over conditioning on full high-dimensional variables. We argue that SCBMs provide a principled alternative to existing causal dimension reduction frameworks such as causal representation learning and causal abstractions.

\end{abstract}

\section{Introduction}\label{sec:intro}
A fundamental aim of scientific inquiry is to uncover and quantify causal relationships among complex phenomena, which often span large spaces, long times, or many individuals. For example, neuroscientists study how neuron clusters respond to tasks \citep{aoi2018model}, while climate scientists examine interactions like the El Niño Southern Oscillation, affecting global weather patterns \citep{timmermann2018nino}.

A popular type of model to formalize causal interactions of random variables is the structural causal model (SCM) \citep{pearl2009causality}. An SCM consists of structural equations
$\X_j := m_j(\X_{i_1},\dots,\X_{i_k},  \boldsymbol{\eta}_j  ),$
one per variable, specifying how each is determined by its \emph{causal parents} and a noise term $\boldsymbol{\eta}_j$ via mechanism $m_j$. Although most SCM applications assume scalar variables, the framework does not require this. However, modeling interactions of high-dimensional vectors $\X_j$ may quickly become infeasible in practice without additional assumptions. Even if the mechanism functions $m_j$ are assumed to be linear and additive, i.e.
$\X_j := \sum_{i \in \pa(j)} \X_i M_{ij} +  \boldsymbol{\eta}_j,$
with matrices $M_{ij}$, the associated regression tasks require large sample sizes and/or a sufficient degree of regularization to yield reliable outcomes in high dimensions. When estimating causal effects in an SCM, the curse of dimensionality is particularly daunting. The estimation may not only require as input the treatment and outcome variables, but also confounding covariates to remove spurious correlations, increasing the dimension of the input space further.

In this work, we investigate SCMs in which variables depend on their parents only through low-dimensional sufficient statistics, or \emph{bottlenecks}: for each parent $\X_i$ of $\X_j$, a deterministic bottleneck function $b_{ij}$ maps $\X_i$ to $\Z_{ij} = b_{ij}(\X_i)$ such that
\begin{align*}
    \X_j := f_j(\Z_{i_1 j},\dots,\Z_{i_k j},  \boldsymbol{\eta}_j  )
\end{align*}
depends on its parents only through their bottlenecks. For linear mechanisms, this requirement translates to the matrices $M_{ij}$ having low rank compared to the dimensions of $\X_i$ and $\X_j$.

The bottleneck assumption is natural for high-dimensional causal interactions, where a child depends not on all information in its parents but on emergent properties---a weighted average, a phase, or a system state. For instance, modelling rainfall over West Africa may require only whether the El Niño Southern Oscillation (ENSO) is in an El Niño or La Niña phase, not the full Pacific temperature distribution; such dimension reduction assumptions have been studied in spatially-aggregated vector autoregressive models of climate teleconnections~\citep{tibau2022spatiotemporal}. Moreover, different children may depend on different aspects of the same parent, calling for \emph{target-dependent} dimension reduction. However, reducing dimensions before estimating causal effects can discard or misidentify important information \citep{wahl2024foundations, ninad2025causal}, and existing \emph{sufficient dimension reduction} approaches \citep{izenman1975reduced, globerson2003sufficient, li2007sparse, li2018sufficient} target single treatment-outcome pairs; with the notable exception of \citet{aoi2018model}, the outcome is typically scalar.

In this work, we introduce structural causal bottleneck models (SCBMs), a class of graphical causal models that enables targeted, mechanism-specific dimension reduction. Our main contributions are:

\begin{itemize}
    \item \textbf{SCBM framework} (\Cref{sec:def}): We formalize the idea that high-dimensional effects are mediated by low-dimensional summary statistics, discuss special cases, and show that bottleneck-equivalent representations are interchangeable for causal inference (Definition~\ref{def.structural_equivalence}).

    \item \textbf{Identifiability} (\Cref{sec:id}; Lemmas~\ref{lem.ident1}--\ref{lem.converse_direction}): We prove that bottleneck variables are identifiable up to invertible reparametrization from observational data, validate this experimentally across linear and nonlinear mechanisms, and illustrate how the equivalence class induces meaningful partitions of high-dimensional input spaces.

    \item \textbf{Effect estimation} (\Cref{sec:est}): We study causal effect estimation under three standard adjustment strategies for linear SCBMs.
    \begin{itemize}
        \item \emph{Instrumental variables} (Proposition~\ref{prop:2sls_nonid}; Theorem~\ref{thm:oracle_2sls}): We identify a population-level non-identifiability of naïve 2SLS when all variables are high-dimensional, and establish a novel identification result that hinges on using bottleneck representations.
        \item \emph{Mediation and backdoor adjustment} (Lemmas~\ref{lem:med-validity},~\ref{lem:bd_validity}; Propositions~\ref{prop:med-efficiency},~\ref{prop:bd_efficiency}): Bottleneck conditioning sets are valid and yield finite-sample efficiency gains over conditioning on full high-dimensional variables, with an explicit efficiency ordering.
    \end{itemize}
\end{itemize}

\paragraph{Related work.} SCBMs sit at the intersection of several active research areas; we discuss connections in detail in Appendix~\ref{sec:rel_work}. Causal representation learning \citep{scholkopf2021causal} and causal abstractions \citep{rubenstein2017causal} both seek low-dimensional causal representations, but target full latent/abstracted SCM recovery rather than mechanism-specific, query-targeted maps tied to a known graph. In the single-variable setting, sufficient dimension reduction \citep{izenman1975reduced} and reduced-rank regression share our low-rank premise but do not address multi-variable graphs or vector-valued outcomes. For high-dimensional IV estimation, prior work addresses instrument-side \citep{bai2010instrumental} or treatment-side \citep{lin2026learning} compression; we instead identify a population-level non-identifiability and show that restoring it requires bottleneck representations on both sides of the treatment. Propensity score matching \citep{rosenbaum1983central} and coarsened exact matching \citep{blackwell2009cem} can also be understood as operating in bottleneck spaces, connecting to our backdoor adjustment results.

\section{Preliminaries \& definitions}\label{sec:def}
Throughout this work, we fix the following notation. $\cG = (\mathcal{V},\mathcal{E})$ will refer to a directed acyclic graph (DAG) with node set $\cV$ and edge set $\cE$. The set of parents of a node $i \in \cV$ will be denoted by $\mathrm{pa}(i)$, the set of its children by $\mathrm{ch}(i)$. The set $\mathcal{V}_{ex}$ is the set of exogenous (i.e. parentless) nodes, while $\mathcal{V}_{end} = \cV \backslash \cV_{ex}$ is the set of endogenous nodes.

We reserve the letter $d$ to refer to dimensions of vector spaces, e.g. $\mathbb R^d$. Unless specified otherwise, we use the word space in the sense of measurable space, i.e. a set endowed with a $\sigma$-algebra. All maps between such spaces are assumed to be measurable. We assume the causal graph $\cG$ is known; causal discovery is outside of the scope of this work.

\begin{definition}[Structural causal bottleneck model] \label{def.structural_bottleneck}
A \emph{structural causal bottleneck model (SCBM)} is a tuple $\mathfrak{C} = \langle \cG, \cX, \cZ, \cH,\boldsymbol{\eta}, \mathcal{B},\mathcal{F}, \X \rangle$ consisting of

\begin{itemize}
    \item a DAG $\cG = (\mathcal{V},\mathcal{E})$, a family of \emph{node spaces} $\mathcal{X} = (\mathcal{X}_i)_{i \in \mathcal{V}}$, and a family of \emph{noise spaces} $\cH = (\mathcal{H}_{i})_{i \in \mathcal{V}}$. 
    \item a family of mutually independent noise vectors $\boldsymbol{\eta} = (\boldsymbol{\eta}_i)_{i \in \mathcal{V}}$, where $\boldsymbol{\eta}_i$ takes values in $\mathcal{H}_i$. 

    \item a family of \emph{bottleneck spaces} $\cZ = (\mathcal{Z}_{j})_{j \in \mathcal{V}_{end}}$. 
 
    \item a family $ \mathcal{B} = (b_{j})_{j \in \cV_{end}} $ of surjective \emph{bottleneck functions} $b_{j}: \prod_{i \in \mathrm{pa}(j)} \mathcal{X}_i  \to \mathcal{Z}_{j}$.

     \item a family $ \mathcal{F} = (f_{j})_{j \in \mathcal{V}_{end}} $ of \emph{effect functions} $f_{j}: \mathcal{Z}_{j} \times \cH_j \to \mathcal{X}_j$.
    \item a family of random vectors $\X = (\X_i)_{i \in \mathcal{V}}$ with joint distribution $P_{\X}$ that solves the structural assignments
        \begin{align*}
            \X_j &:= \boldsymbol{\eta}_j, \qquad &j \in \mathcal{V}_{ex}, \\
            \X_j &:= f_j\left( b_{j}\left((\X_i)_{i \in \mathrm{pa}(j)}\right),\boldsymbol{\eta}_j\right), \qquad &j \in \mathcal{V}_{end}.
        \end{align*}

\end{itemize}

\end{definition}

Any SCBM $\mathfrak{C}$ straightforwardly induces an SCM $\mathfrak{A}(\mathfrak{C})$ in the classical sense \citep{pearl2009causality} by defining the mechanism functions as $m_j\left((\mathbf{X}_i)_{i\in\mathrm{pa}(j)}, \boldsymbol{\eta}_j\right) := f_j\left(b_j\big((\mathbf{X}_i)_{i\in\mathrm{pa}(j)}\right), \boldsymbol{\eta}_j\big)$. We call this SCM the induced SCM. Conversely, every SCM is the induced SCM of a trivial SCBM with $\mathcal{Z}_j = \prod_{i\in\mathrm{pa}(j)} \mathcal{X}_i$, $b_j = \mathrm{id}$ and $f_j = m_j$. Definition 2.1 by itself therefore imposes no restriction; dimension reduction enters only through the minimality condition of Lemma~\ref{lem.converse_direction}.

\begin{definition} \label{def.structural_equivalence}
    We call two SCBMs \emph{structurally equivalent} if their induced SCMs coincide.
\end{definition}

Since all interventional distributions and the observational distribution of $\X$ are fully determined by the induced SCM, structurally equivalent SCBMs share the same interventional distributions (interventional equivalence) and observational distribution (observational equivalence).
In particular, every SCBM is structurally equivalent to the trivial SCBM of its induced SCM. The bottlenecks of interest are the minimal members of this equivalence class, which we characterize in \cref{sec:id}.

For the remainder of this work, we consider so-called \emph{factored SCBMs}, where each bottleneck can be subdivided into separate bottlenecks for each parent.

\begin{definition}[Factored SCBMs]\label{def.factored}
We call an SCBM $\mathfrak{C}$ \emph{factored} if
\begin{itemize}
    \item [(a)] Each bottleneck is \emph{factored} in the sense that $\cZ_j$ and $b_j$ can be decomposed as $\cZ_j = \prod_{i \in \pa(j)} \cZ_{ij}$ and
\begin{align*}
     b_j: \prod_{i \in \pa(j)} \cX_{i} &\to \prod_{i \in \pa(j)} \cZ_{ij} \\  b_j((\x_i)_{i \in \pa(j)} ) &= (b_{ij}(\x_i))_{i \in \pa(j)} 
\end{align*}
with $\dim \cZ_{ij} \leq d_i$.
In other words, we assume that there is a separate bottleneck space for every parent of the endogenous node $j$. 
\item[(b)] Each $\cX_j$ is a vector space, and coincides with the noise space $\cH_j = \cX_j$. Each effect function $f_j$ is factored in the sense that there is a family of  maps $f_{ij}: \cZ_{ij} \to \cX_j$ such that
\begin{align*}
    \X_j := \sum_{i \in \pa(j)} f_{ij}(b_{ij}(\X_i)) + \boldsymbol{\eta}_j.
\end{align*}

\end{itemize}

We illustrate Definition~\ref{def.factored} in \Cref{fig:factored}; each edge in the DAG is endowed with a separate bottleneck.
\end{definition}
\begin{wrapfigure}{r}{0.44\textwidth}
    \centering
    \resizebox{.75\linewidth}{!}{
    \tikz[latent/.append style={minimum size=1cm},obs/.append style={minimum size=1cm},det/.append style={minimum size=1.15cm}, wrap/.append style={inner sep=2pt}, plate caption/.append style={below left=5pt and 0pt of #1.south east}, on grid]{
        \node[latent](x_g){$\X_g$};
        \node[obs, right=2cm of x_g](z_g){$\Z_{gj}$};
        \node[latent, below=1.5cm of x_g](x_h){$\X_h$};
        \node[obs, right=2cm of x_h](z_h){$\Z_{hj}$};
        \node[latent, below=1.5cm of x_h](x_i){$\X_i$};
        \node[obs, right=2cm of x_i](z_i){$\Z_{ij}$};
        \node[latent, right=2cm of z_h](x_j){$\X_j$};

        \edge[-{Stealth[length=2mm, width=1.5mm]}]{x_g}{z_g};
        \edge[-{Stealth[length=2mm, width=1.5mm]}]{z_g}{x_j};
        \edge[-{Stealth[length=2mm, width=1.5mm]}]{x_h}{z_h};
        \edge[-{Stealth[length=2mm, width=1.5mm]}]{z_h}{x_j};
        \edge[-{Stealth[length=2mm, width=1.5mm]}]{x_i}{z_i};
        \edge[-{Stealth[length=2mm, width=1.5mm]}]{z_i}{x_j};
    }
    }
    \caption{A \emph{factored} SCBM.}
    \label{fig:factored}
\end{wrapfigure}
\paragraph{Equivalence relations induced by bottlenecks.}
If $(i,j)$ is an edge in $\cG$, then $b_{ij}$ can be considered a quotient map that induces an equivalence relation on $\cX_i$: two states $\mathbf{x}, \mathbf{x}' \in \cX_i$ are \emph{bottleneck-equivalent relative to child} $j$ if $b_{ij}(\mathbf{x}) = b_{ij}(\mathbf{x}')$, which implies they are \emph{causally equivalent} w.r.t.\ $\X_j$ in the sense of \citet{chalupka2017causal}:
\begin{align*}
    P(\X_j \mid \mathrm{do}(\X_i = \mathbf{x})) = P(\X_j \mid \mathrm{do}(\X_i = \mathbf{x}')).
\end{align*}
We discuss special cases of SCBMs, including interactions between random fields, temporal processes in the frequency domain, and a connection to the information bottleneck of \citet{tishby2000information}, in Appendix~\ref{sec:spec}.


\section{Identifiability of SCBMs}\label{sec:id}
We study identifiability of SCBMs, formalised as recovering bottleneck variables up to bijection from observational data---the degree of identifiability sufficient for causal effect estimation and intervention selection (cf.\ \Cref{sec:def}). We establish this result theoretically and validate it experimentally, illustrating how recovered bottlenecks yield meaningful partitions of the input space.

\subsection{Identifiability theory}
Inserting invertible maps on bottleneck spaces yields a structurally equivalent SCBM (Lemma~\ref{lem.ident1}); the converse holds for additive-noise SCBMs (Lemma~\ref{lem.converse_direction}). Proofs are in Appendix~\ref{app:proofs}.

\begin{lemma}[Bijective reparametrization preserves structural equivalence]
\label{lem.ident1}
    Let $\mathfrak{C} = \langle \cG, \cX, \cZ, \cH, \boldsymbol{\eta}, \mathcal{B},\mathcal{F}, \X \rangle$ be an SCBM. Assume that there are invertible maps $\psi_j: \cZ_j \to \cZ'_j$ for every endogeneous node $j \in \cV_{end}$ to some other spaces $\cZ'=(\cZ'_j)_{j \in \cV_{end}}$, and consider the functions $b'_j = \psi_j \circ b_j$ and $f'_j( \cdot \ ,\ \cdot ) = f_j(\psi_j^{-1}( \ \cdot \ ),\ \cdot)$. Then $\mathfrak{C}' = \langle \cG, \cX, \cZ', \boldsymbol{\eta}, \mathcal{B}',\mathcal{F}', \X \rangle$ with $\cF' = (f'_j)_j$ and  $\cB' = (b'_j)_j$ is structurally equivalent to $\mathfrak{C}$.
\end{lemma}

\begin{lemma}[Structural equivalence implies bijective bottlenecks]
\label{lem.converse_direction}
    Let $\mathfrak{C} = \langle \cG, \cX, \cZ, \cH, \boldsymbol{\eta}, \mathcal{B},\mathcal{F}, \X \rangle$ and $\mathfrak{C}' = \langle \cG, \cX, \cZ', \cH', \boldsymbol{\eta}, \mathcal{B}',\mathcal{F}', \X \rangle$  be two SCBMs with additive noises, i.e. $f_j(\mathbf{z}_{\mathrm{pa}(j)},\eta_j) = \tilde{f}_j(\mathbf{z}_{\mathrm{pa}(j)})+ \eta_j$, and similarly for $f'_j$. Assume that the functions $\tilde{f}_{j}: \cZ_{j} \to \cX_j, \ \tilde{f}'_{j}: \cZ_{j}' \to \cX_j$ are almost surely injective. If $\mathfrak{C}$ and $\mathfrak{C}'$ are structurally equivalent, there is an invertible function $\psi_j: \cZ_j \to \cZ_j'$ such that $b' = \psi \circ b$ and $f_j' = f_j \circ \psi^{-1}$.
\end{lemma}

Injectivity of the effect functions is a minimality assumption: together with surjectivity of $b_j$, it forces the bottleneck to retain exactly the information about the parents that the mechanism uses. In particular, it excludes the trivial SCBM with $b_j = \mathrm{id}$ unless $\tilde f_j$ is itself injective, in which case no compression is possible. In the linear case, it guarantees $\mathrm{rank}(f_j \circ b_j) = \dim \mathcal{Z}_j$, so the bottleneck dimension equals the rank of the mechanism. We refer to SCBMs satisfying this condition as \emph{minimal}. The linear SCBMs in \cref{sec:est} are minimal by construction, since each $W^f_{ij}$ has full row rank. Within this class, bottlenecks are unique up to a change of basis (cf. Lemma~\ref{lem.ident1}).

\subsection{Estimating bottlenecks}\label{ssec:estimating}
\begin{figure*}[t]
    \centering
    \includegraphics[width=\linewidth]{figures/id_combo.png}
    \    \caption{\textbf{Left:} Mean $\pm$ std of avg. $R^2$ for bottleneck recovery across settings; top row linear, bottom row nonlinear SCBMs. High scores confirm recovery up to bijection. \textbf{Right:} Learned bottleneck spaces $\hat{\mathcal{Z}}$ vs.\ ground-truth $\mathcal{Z}$, related by an affine map (linear) or smooth bijection (nonlinear).}
    \label{fig:exp_id}
\end{figure*}
Given a graph, estimating each bottleneck function $b_{ij}$ amounts to fitting a standard regressor from $\X_i$ to $\X_j$ with a rank constraint on the bottleneck dimension, using any valid conditioning set for the effect $\X_i \to \X_j$. In the linear case this reduces to a rank-constrained matrix factorization; in the nonlinear case we use an encoder-decoder architecture. By Lemma~\ref{lem.converse_direction}, the recovered bottleneck satisfies $\hat{b}_{ij} = \psi_{ij} \circ b_{ij}$ for some invertible $\psi_{ij}$. Full details are in Appendix~\ref{app:ssec:estim}; code is available at \url{https://github.com/simonbing/StructuralCausalBottleneckModels}.

We sample SCBMs with varying $|\mathcal{V}|$, $d_{\X}$, $d_{\Z}$, and mechanism type (linear/nonlinear), draw $n$ samples, estimate all bottlenecks, and measure recovery via the average bidirectional $R^2$ between each $\Z_{ij}$ and $\hat{\Z}_{ij}$; bidirectionality is required to test equivalence up to bijection rather than mere surjection. Default settings are $|\mathcal{V}| = 10$, $d_{\X} = 5$, $d_{\Z} = 2$, $n = 30000$ (linear) / $50000$ (nonlinear); details are in Appendix~\ref{app:ssec:datagen}.

In linear SCBMs, recovery is successful across all settings, with performance saturating around $n = 10000$ and remaining robust as $|\mathcal{V}|$ scales, indicating minimal error propagation. For nonlinear SCBMs, scores are lower but degrade gracefully: larger $|\mathcal{V}|$ reduces performance without causing catastrophic failure,\footnote{Graphs with $|\mathcal{V}|=100$ have $\sim$3500 edges, each requiring a separate network; in practice one estimates only the bottlenecks needed for a specific query.} and larger $n$ largely compensates. We also illustrate the learned bottleneck spaces in \Cref{fig:exp_id}: the invertible transformation $\psi$ (cf.\ Lemma~\ref{lem.converse_direction}) manifests as an affine map in the linear case and a smooth bijection in the nonlinear case. Additional visualizations are in Appendix~\ref{app:sec:id_viz}.

\subsection{Estimation under misspecification}
\label{ssec:misspec}
\begin{wrapfigure}{r}{0.48\textwidth}
    \centering
    \vspace{-21pt}
    \includegraphics[width=0.62\linewidth]{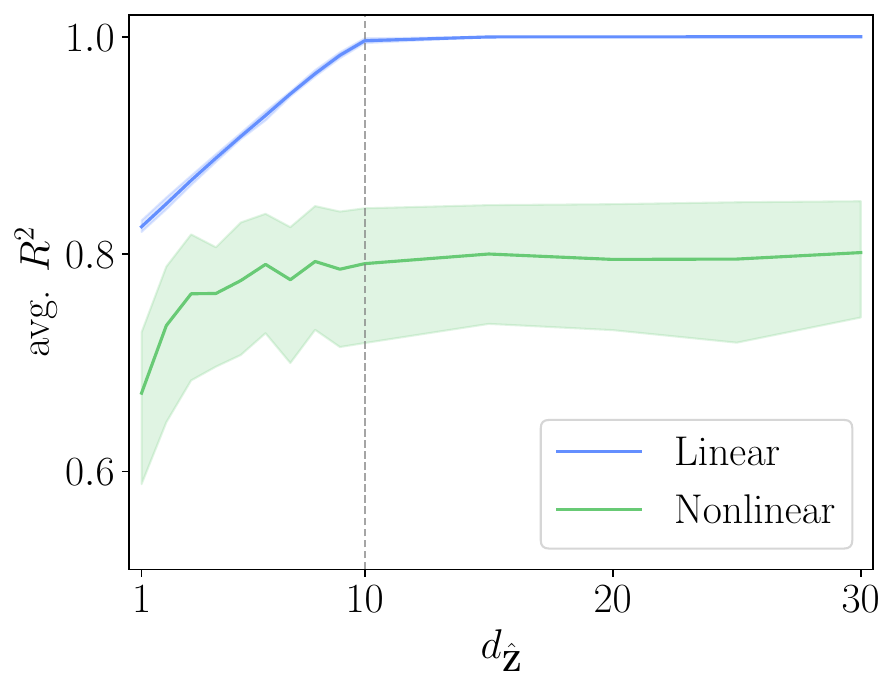}
    \caption{Mean $R^2$ $\pm$ 95\% CI for varying assumed bottleneck dimension $d_{\hat{\Z}}$; true $d_{\Z}$ indicated by dashed line.}
    \vspace{-5pt}
    \label{fig:lin_misspec}
\end{wrapfigure}
We now drop the assumption that $d_{\Z}$ is known and study how recovery degrades when the assumed bottleneck dimension $d_{\hat{\Z}}$ is misspecified. We sweep $d_{\hat{\Z}}$ from $1$ to $d_{\X}$ on a cause-effect pair with $d_{\Z} = 10$ ($d_{\X} = 50$ linear, $d_{\X} = 100$ nonlinear). As shown in \Cref{fig:lin_misspec}, performance saturates at $d_{\hat{\Z}} = d_{\Z}$: over-compression loses information while under-compression does not hurt. This is a key practical advantage over CRL, where both under- \emph{and} overestimating the latent dimension compromises identifiability.

\subsection{Partitioning and intervention selection}
We illustrate identifiability concretely using the three-node SCBM in \Cref{fig:iv-selection}, where $\X_1 \in \mathbb{R}^2$ is a shared root with two children. The two bottlenecks capture fundamentally different properties of the same input: $b_{12}$ depends on the Euclidean norm, $b_{13}$ on the quadrant. Each bottleneck induces an equivalence relation on $\cX_1$ via its level sets; since level sets are invariant to bijections, this partition is well-defined up to the identifiability class. The learned bottlenecks $\hat{b}_{12}$ and $\hat{b}_{13}$ recover the ground-truth partitions (concentric rings and quadrants, respectively), and projecting onto the output spaces reveals which input regions produce which downstream effects. Since bottleneck-equivalent inputs are causally equivalent (cf.\ \Cref{sec:def}), a practitioner need only consider one representative per class---reducing $\cX_1$ to a one-dimensional scale for $\X_2$ and four classes for $\X_3$ without requiring interventional data.

\begin{figure}[t]
    \centering
    %
    \begin{minipage}[c]{0.38\linewidth}
        \begin{minipage}[t]{\linewidth}
            \begin{align*}
                \mathbf{X}_1 &= \boldsymbol{\eta}_1, \\
                \mathbf{X}_2 &= f_{12}(b_{12}(\mathbf{X}_1)) + \boldsymbol{\eta}_2, \\
                \mathbf{X}_3 &= f_{13}(b_{13}(\mathbf{X}_1)) + \boldsymbol{\eta}_3,
            \end{align*}
            where
            \begin{align*}
                b_{12} &:= \sqrt{{\*X^{(1)}_1}^2 + {\*X^{(2)}_1}^2},\\
                b_{13} &:= \begin{cases}
                    0 & \textrm{if } \*X^{(1)}_1 \leq 0, \*X^{(2)}_1 \leq 0 \\
                    1 & \textrm{if } \*X^{(1)}_1 \leq 0, \*X^{(2)}_1 > 0 \\
                    2 & \textrm{if } \*X^{(1)}_1 > 0, \*X^{(2)}_1 \leq 0 \\
                    3 & \textrm{if } \*X^{(1)}_1 > 0, \*X^{(2)}_1 > 0
                \end{cases}.
            \end{align*}
        \end{minipage}
    \end{minipage}
    \hfill
    %
    \begin{minipage}[c]{0.58\linewidth}
        \centering
        \includegraphics[width=0.95\linewidth]{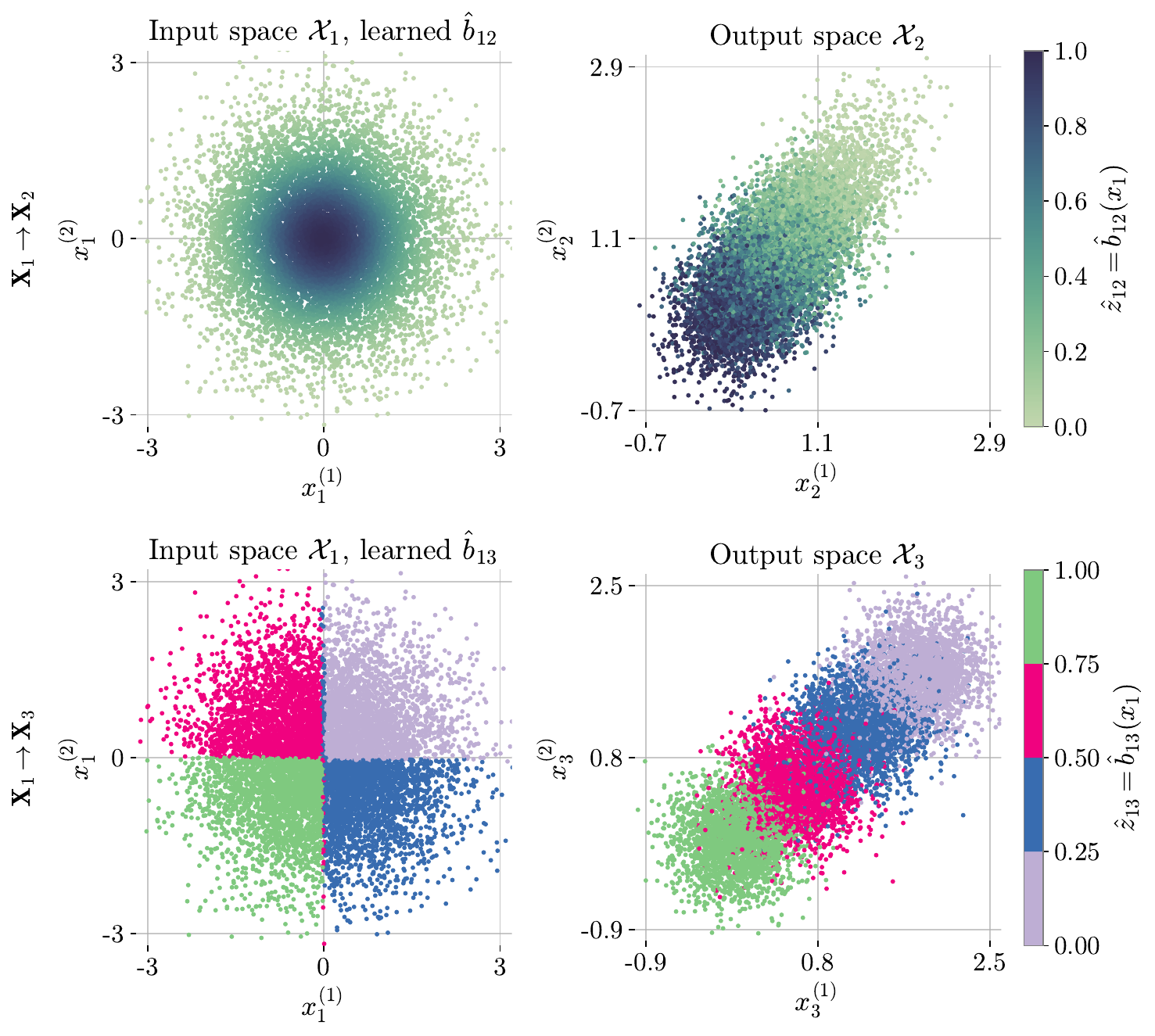}
    \end{minipage}
    \caption{\textbf{Left:} Structural equations of a three-node SCBM with shared root $\X_1 \in \mathbb{R}^2$. \textbf{Right:} Input-space partitions (left column) and output spaces (right column) induced by $\hat{b}_{12}$ (top) and $\hat{b}_{13}$ (bottom). Colors denote bottleneck-equivalent, hence causally equivalent, states. The two bottlenecks induce fundamentally different partitions of the same $\X_1$---concentric rings vs.\ quadrants---invariant to bijective reparametrizations (Lemma~\ref{lem.converse_direction}). In the output spaces, only transitions across partition boundaries produce changes in the downstream distribution.}
    \label{fig:iv-selection}
\end{figure}

\section{Effect estimation in SCBMs}\label{sec:est}
High-dimensional variables introduce qualitatively new challenges for causal effect estimation: standard identification strategies can break down, and even when identification holds, the curse of dimensionality inflates finite-sample variance. We study three causal structures in which this occurs and show that the bottleneck structure of SCBMs provides a principled remedy in each case. In the \emph{instrumental variable} (IV) setting, na\"{i}ve 2SLS is generically non-identified when all variables are $d$-dimensional; oracle bottleneck representations restore identifiability. For \emph{mediation} and \emph{backdoor adjustment} without unobserved confounding, identification holds but bottleneck conditioning sets yield efficiency gains.

Throughout this section, we restrict to linear SCBMs, in which bottleneck and effect functions are linear maps, $b_{ij}(\mathbf{x}) = \mathbf{x}W^b_{ij}$ and $f_{ij}(\mathbf{z}) = \mathbf{z}W^f_{ij}$ with $W^b_{ij} \in \mathbb{R}^{d_i \times r}$ and $W^f_{ij} \in \mathbb{R}^{r \times d_j}$, where $r := \dim \mathcal{Z}_{ij}$, so that the mechanism matrix factors as $M_{ij} = W^b_{ij}W^f_{ij}$. The bottleneck assumption then reduces to a rank condition: $M_{ij}$ factors through an $r$-dimensional bottleneck space if and only if $\mathrm{rank}(M_{ij}) \le r$, and the bottleneck is minimal in the sense of \cref{sec:id} if and only if $W^b_{ij}$ and $W^f_{ij}$ both have rank $r$, i.e.\ $\mathrm{rank}(M_{ij}) = r$. Minimal factorizations are unique up to $W^b_{ij} \mapsto W^b_{ij}G$, $W^f_{ij} \mapsto G^{-1}W^f_{ij}$ for invertible $G \in \mathbb{R}^{r\times r}$, the linear instance of Lemmas~\ref{lem.ident1} and~\ref{lem.converse_direction}. Our target is the direct effect $M_{XY}$.

It is useful to separate four ingredients of our framework. (i) The \emph{modelling assumption} is structural: data is generated by a factored SCBM with additive noise and minimal bottlenecks (Definitions~\ref{def.structural_bottleneck},~\ref{def.factored}; Lemma~\ref{lem.converse_direction}) over a known graph $\mathcal{G}$, here specialized to the linear case. (ii) The \emph{identifiability target} is the family of bottleneck functions $b_{ij}$ up to invertible reparametrizations $\psi_{ij}$ (Lemma~\ref{lem.converse_direction}), i.e.\ what the observational distribution determines in principle. (iii) The \emph{estimation procedure} (\cref{ssec:estimating}, \cref{alg:estim}) consists of rank-constrained regression or an encoder--decoder with assumed bottleneck dimension $d_{\hat{\mathbf{Z}}}$ and a chosen conditioning set; these are choices of the practitioner rather than assumptions on the model, and may be misspecified (\cref{ssec:misspec}). (iv) The \emph{oracle quantities} are the true bottleneck matrices $W^b_{ij}$, to which we assume access throughout this section. Since all results below are invariant under the reparametrization $W^b_{ij} \mapsto W^b_{ij}G$, the oracle only needs to supply bottlenecks up to the identifiability class in (ii). Working with oracle bottlenecks isolates their statistical benefit from the cost of learning them, which we quantify in Appendix~\ref{app:sample_cost}. Accordingly, experiments in this section use oracle bottlenecks unless stated otherwise, whereas those in \cref{sec:id} use estimated ones. Full proofs are deferred to Appendix~\ref{app:est_proofs}; proof sketches are presented for central results.

\subsection{Unobserved confounding}\label{ssec:iv}
When treatment $\*X$ and outcome $\*Y$ are confounded by an unobserved variable $\*U$, standard regression is biased and an instrumental variable $\*I$ can enable identification. In the classical scalar setting, two-stage least squares (2SLS) with a valid instrument recovers the causal effect consistently. 
In the high-dimensional SCBM setting, however, the bottleneck structure renders $M_{IX}$ rank-deficient, so naïve 2SLS is underidentified even with a valid instrument. This is an instance of the known rank-condition failure of linear IV, under which 2SLS recovers at best a projection of the effect \citep{pfister2022identifiability,ailer2023sequential}. Since here this failure stems from the mechanism rather than from redundant instruments, it cannot be resolved by instrument selection \citep{brito2002generalized}. Whereas prior remedies rely on additional experiments or sparsity, we show that a bottleneck factorization of the treatment–outcome mechanism restores identification given the bottlenecks (\cref{thm:oracle_2sls}).

\begin{proposition}[Non-identifiability of naïve 2SLS in linear SCBMs]
\label{prop:2sls_nonid}
Let $\*I, \*X, \*Y, \*U \in \mathbb{R}^{d}$ be the instrument, treatment, outcome,
and unobserved confounder, respectively, related by the linear structural equations
\begin{align}
    \*I &= \boldsymbol{\eta}_I, \label{eq:scm_i}\\
    \*U &= \boldsymbol{\eta}_U, \label{eq:scm_u}\\
    \*X &= \*I\, M_{IX} + \*U\, M_{UX} + \boldsymbol{\eta}_X, \label{eq:scm_x}\\
    \*Y &= \*X\, M_{XY} + \*U\, M_{UY} + \boldsymbol{\eta}_Y, \label{eq:scm_y}
\end{align}
where the mechanism matrix $M_{IX} \in \mathbb{R}^{d \times d}$ admits the SCBM
factorization
\[
    M_{IX} \;=\; W^{b}_{IX}\,W^{f}_{IX},
    \qquad W^{b}_{IX} \in \mathbb{R}^{d \times r},\quad
    W^{f}_{IX} \in \mathbb{R}^{r \times d},
\]
with $W^{b}_{IX}$ and $W^{f}_{IX}$ each of rank $r$, so that
$\operatorname{rank}(M_{IX}) = r < d$.
Assume $\*I \perp (\*U,\,\*\eta_X,\,\*\eta_Y)$ and
$\Sigma_I := \mathbb{E}[\*I^{\!\top}\*I] \succ 0$.
Then the population 2SLS estimator for $M_{XY}$ has no unique solution.
\end{proposition}


The non-identifiability is a direct consequence of the bottleneck structure: the rank deficiency of $M_{IX}$ makes the 2SLS Gram matrix $G$ singular with a $(d-r)$-dimensional null space, so the normal equations are satisfied by infinitely many solutions. The following remark characterizes what any consistent regularized variant of 2SLS actually converges to.

\begin{remark}[Limiting bias of regularized 2SLS]
\label{rem:proj_bias}
Any consistent regularized variant of 2SLS whose regularization is
shrunk to zero---such as ridge 2SLS with penalty $\lambda \to 0^{+}$---converges to the minimum-norm solution of $G\,\Gamma \;=\; G\,M_{XY}$,
namely
\[
    \Gamma^{*} \;=\; G^{+}G\,M_{XY}
               \;=\; P_{\operatorname{row}(M_{IX})}\,M_{XY},
\]
where $P_{\operatorname{row}(M_{IX})}$ is the orthogonal projector onto
$\operatorname{row}(M_{IX})$.
In this SCBM, $\operatorname{row}(M_{IX}) = \operatorname{row}(W^{f}_{IX})$
and $\operatorname{col}(M_{XY}) = \operatorname{col}(W^{b}_{XY})$
are generically distinct $r$-dimensional subspaces of $\mathbb{R}^{d}$,
so $\Gamma^{*} \neq M_{XY}$ with probability one whenever the weight
matrices are drawn from any distribution that is absolutely continuous
with respect to Lebesgue measure on $\mathbb{R}^{d \times r}$.
\end{remark}

This makes explicit, in the SCBM, that regularized 2SLS converges not to $M_{XY}$ but to its projection onto $\operatorname{row}(M_{IX})$---a biased target whenever the instrument--treatment and treatment--outcome subspaces differ, which holds generically. The root cause is the mismatch between the $d$-dimensional regression space and the effective $r$-dimensional signal. The minimal bottleneck factorization makes this low-rank structure explicit: projecting both the instrument and the treatment onto their respective $r$-dimensional bottleneck spaces yields a well-posed first stage and restores identifiability.

\begin{theorem}[Identification via oracle bottleneck 2SLS]
\label{thm:oracle_2sls}
Under the setup of Proposition~\ref{prop:2sls_nonid}, additionally let
$M_{XY} = W^b_{XY}W^f_{XY}$ with $W^b_{XY} \in \mathbb{R}^{d \times r}$,
$W^f_{XY} \in \mathbb{R}^{r \times d}$, each of rank $r$.
Suppose the rank condition
\begin{equation}
    \operatorname{rank}\!\left(W^f_{IX}\,W^b_{XY}\right) = r
    \label{eq:rc}
    \tag{RC}
\end{equation}
holds. Given oracle access to $W^b_{IX}$ and $W^b_{XY}$, define the
$r$-dimensional bottleneck representations
\[
    \*Z_{IX} \;:=\; \*I\,W^b_{IX},
    \qquad
    \*Z_{XY} \;:=\; \*X\,W^b_{XY}.
\]
The oracle bottleneck 2SLS procedure regresses $\*Z_{XY}$ on $\*Z_{IX}$
(first stage) to obtain fitted values $\hat{\*Z}_{XY}$, then regresses
$\*Y$ on $\hat{\*Z}_{XY}$ (second stage) to obtain $\hat{\Phi}$, and
forms the estimator
\[
    \widehat{M}_{XY} \;:=\; W^b_{XY}\,\hat{\Phi}.
\]
Then $\widehat{M}_{XY} = M_{XY}$.
\end{theorem}

\begin{proof}[Proof sketch]
Since $W^b_{IX}$ has rank $r$ and $\Sigma_I \succ 0$, the first-stage Gram matrix $\Sigma_{Z_{IX}} \succ 0$ gives a unique first stage; rank condition~\eqref{eq:rc} then ensures the second-stage Gram matrix is positive definite, yielding $\hat{\Phi} = W^f_{XY}$. Composing gives $\widehat{M}_{XY} = W^b_{XY}\hat{\Phi} = M_{XY}$.
\end{proof}

The two oracle bottlenecks in \cref{thm:oracle_2sls} differ in how they can be obtained. Since the instrument is unconfounded, the first-stage regression of $\mathbf{X}$ on $\mathbf{I}$ identifies $M_{IX}$, and a rank-$r$ factorization yields $W^b_{IX}$ up to the reparametrization $G$. The bottleneck $W^b_{XY}$, by contrast, belongs to the confounded mechanism $\mathbf{X} \to \mathbf{Y}$. The IV moment conditions only determine its projection $P_{\mathrm{row}(M_{IX})}M_{XY}$ (Remark~\ref{rem:proj_bias}), so it cannot be learned from the observational data alone. In practice, $W^b_{XY}$ must therefore come from a transfer environment in which the $\mathbf{X} \to \mathbf{Y}$ mechanism is invariant, such as a pilot RCT intervening on $\mathbf{X}$ (Appendix~\ref{app:transfer_iv}).

Related approaches are \citet{bai2010instrumental}, who address finite-sample many-instruments bias by compressing many instruments via PCA, and \citet{lin2026learning}, who identify a backdoor bias from omitting instrument-relevant treatment directions and recover these via SVD. Both works face a bias rather than a population-level non-identifiability, and neither assumes a bottleneck factorization of the treatment--outcome mechanism $M_{XY}$; see Appendix~\ref{sec:rel_work} for a detailed comparison.

\subsection{No unobserved confounding}\label{ssec:adj}
When all confounders are observed, the causal effect is identifiable by standard regression. The question is one of \emph{efficiency}: conditioning on high-dimensional observed variables inflates finite-sample variance, and lower-dimensional bottleneck representations offer a more parsimonious alternative. The results below establish which bottleneck representations are valid adjustment variables in each setting, and show that replacing the observed variable with its bottleneck reduces asymptotic variance. We consider two settings: mediator adjustment and backdoor adjustment; see \Cref{fig:bottleneck_adjustment}.

\subsubsection{Mediator adjustment}\label{sssec:med}
Let $\*X \in \mathbb{R}^{d}$ be the treatment, $\*M \in \mathbb{R}^{d}$ an observed mediator,
and $\*Y \in \mathbb{R}^{d}$ the outcome. The linear SCBM structural equations are
\begin{align}
  \*X &= \boldsymbol{\eta}_X, \label{eq:med-x}\\
  \*M &= \*X M_{XM} + \boldsymbol{\eta}_M, \label{eq:med-m}\\
  \*Y &= \*X M_{XY} + \*M M_{MY} + \boldsymbol{\eta}_Y, \label{eq:med-y}
\end{align}
where $M_{ij} = W^b_{ij}W^f_{ij} \in \mathbb{R}^{d\times d}$ with $\operatorname{rank}(M_{ij})=r\ll d$,
and $\Sigma_{\boldsymbol{\eta}_X},\Sigma_{\boldsymbol{\eta}_M},\Sigma_{\boldsymbol{\eta}_Y}\succ 0$ are the noise covariances,
all mutually independent.
The target is the \emph{direct effect} $M_{XY}$.
Define the \emph{outcome-side mediator bottleneck}
$\*Z_{MY} := \*M W^b_{MY} \in \mathbb{R}^{n\times r}$.

The graph contains two bottleneck representations on the indirect path $\*X \to \*M \to \*Y$: one on the $\*X \to \*M$ edge and one on the $\*M \to \*Y$ edge. The following lemma resolves which of these bottlenecks can replace the full mediator $\*M$ as a conditioning variable.

\begin{lemma}[Validity of mediator bottleneck adjustment]
\label{lem:med-validity}
Let $\hat{M}_{XY}^{S}$ be the OLS estimator of $M_{XY}$ from the regression of $\*Y$
on $(\*X,\, S)$.
Both $S = \*M$ and $S = \*Z_{MY}$ yield consistent estimators:
\[
  \hat{M}_{XY}^{M} \xrightarrow{p} M_{XY}, \qquad
  \hat{M}_{XY}^{Z_{MY}} \xrightarrow{p} M_{XY}.
\]
\end{lemma}


Both $\*M$ and $\*Z_{MY}$ are therefore consistent conditioning variables. Since $\*Z_{MY}$ is $r$-dimensional while $\*M$ is $d$-dimensional with $r \ll d$, conditioning on $\*Z_{MY}$ requires far fewer effective degrees of freedom. This translates into a reduction in asymptotic variance.

\begin{proposition}[Efficiency gain of mediator bottleneck adjustment]
\label{prop:med-efficiency}
Under the linear SCBM, the asymptotic covariance of $\hat{M}_{XY}^{S}$
is $\Phi_S := Q_S^{-1}\Sigma_{\boldsymbol{\eta}_Y}Q_S^{-1}$, where
$Q_S := \operatorname{Cov}(\tilde{\*X}^S)$ denotes the covariance of
the Frisch--Waugh--Lovell \citep{frisch1933partial,lovell1963seasonal} (FWL) residual of $\*X$ after projecting out $S$.  Then
\[
  \Phi_{Z_{MY}} \preceq \Phi_{M}.
\]
\end{proposition}


We provide experimental evidence that underlines the efficiency gains of $\*Z_{MY}$ in \Cref{fig:bottleneck_adjustment} (left).

The efficiency gain of $\*Z_{MY}$ over $\*M$ may suggest that any low-dimensional bottleneck along the indirect path is beneficial; in particular, one might ask whether the treatment-side bottleneck $\*Z_{XM} := \*X W^b_{XM}$ is also a valid adjustment variable. It is not: since $\*Z_{XM}$ is a linear function of $\*X$, regressing $\*Y$ on $(\*X,\,\*Z_{XM})$ is equivalent to regressing on $\*X$ alone, so the estimator converges to the total effect $M_{XY} + M_{XM}M_{MY}$,
instead of the direct effect $M_{XY}$. A valid conditioning variable must lie on the $\*M \to \*Y$ side of the indirect path---as $\*Z_{MY}$ does.
\begin{figure}
    \begin{minipage}[c]{0.48\linewidth}
        \centering
        \textbf{Mediator adjustment}
        \newline
            \begin{minipage}[c]{0.32\linewidth}
            \resizebox{\linewidth}{!}{
                    \tikz[latent/.append style={minimum size=0.85cm},obs/.append style={minimum size=0.85cm},det/.append style={minimum size=1.15cm}, wrap/.append style={inner sep=2pt}, plate caption/.append style={below left=5pt and 0pt of #1.south east}, on grid]{
                        \node[latent](x){$\*X$};
                        \node[latent, right=2cm of x](y){$\*Y$};
                        \node[latent, above=1.73cm of x, xshift=1cm, fill=scbmblue!20, draw=scbmblue](m){$\*M$};

                        \edge[-{Stealth[length=2mm, width=1.5mm]}]{x}{m,y};
                        \draw[scbmred, -{Stealth[length=2mm, width=1.5mm]}] (m) -- node[pos=0.35, right, font=\small, text=scbmred] {$\Z_{MY}$} (y);
                    }
                }
            \end{minipage}
            \hfill
            \begin{minipage}[c]{0.64\linewidth}
                \centering
                \includegraphics[width=\linewidth]{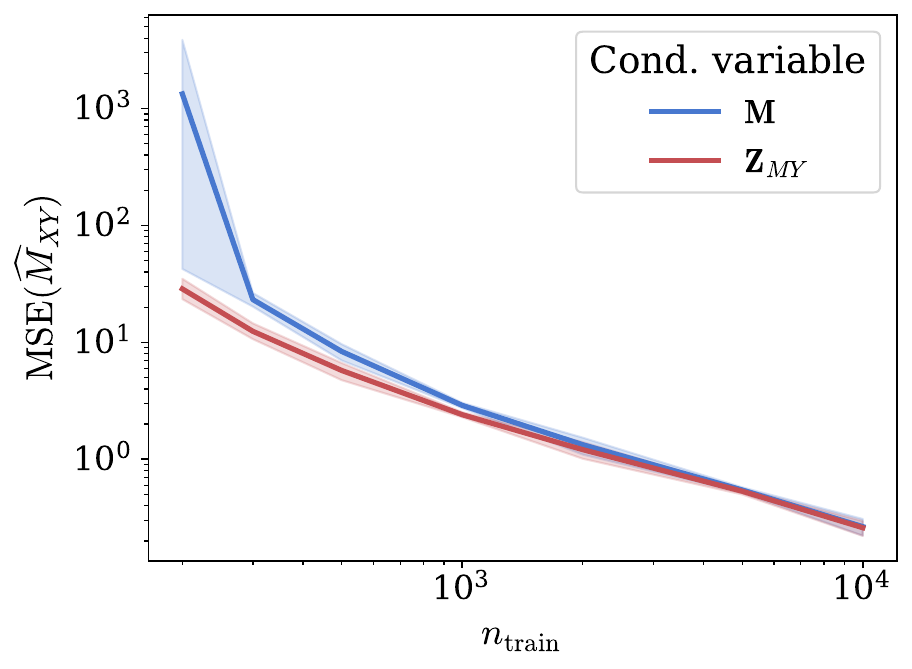}
            \end{minipage}
            \hfill
    \end{minipage}
    \begin{minipage}[c]{0.48\linewidth}
        \centering
        \textbf{Backdoor adjustment}
        \newline
            \begin{minipage}[c]{0.32\linewidth}
                \resizebox{\linewidth}{!}{
                    \tikz[latent/.append style={minimum size=0.85cm},obs/.append style={minimum size=0.85cm},det/.append style={minimum size=1.15cm}, wrap/.append style={inner sep=2pt}, plate caption/.append style={below left=5pt and 0pt of #1.south east}, on grid]{
                        \node[latent](x){$\*X$};
                        \node[latent, right=2cm of x](y){$\*Y$};
                        \node[latent, above=1.73cm of x, xshift=1cm, fill=scbmblue!20, draw=scbmblue](c){$\*C$};

                        \edge[-{Stealth[length=2mm, width=1.5mm]}]{c}{x};
                        \draw[scbmred, -{Stealth[length=2mm, width=1.5mm]}] (c) -- node[pos=0.35, right, font=\small, text=scbmred] {$\Z_{CY}$} (y);
                        \edge[-{Stealth[length=2mm, width=1.5mm]}]{x}{y};
                    }
                }
            \end{minipage}
            \hfill
            \begin{minipage}[c]{0.64\linewidth}
                \centering
                \includegraphics[width=\linewidth]{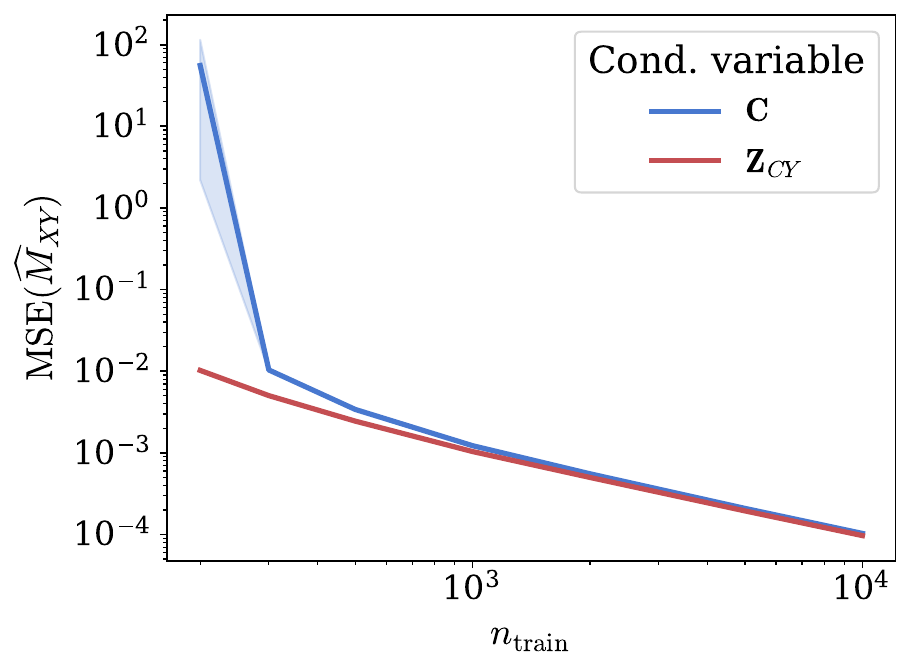}
            \end{minipage}
    \end{minipage}
    
    \caption{\textbf{Left:} Mediator adjustment; \textbf{Right:} backdoor adjustment. MSE of $\widehat{M}_{XY}$ vs.\ $n_{\mathrm{train}}$ (log-log) for conditioning on the full high-dimensional variable (\textcolor{scbmblue}{blue}, $d=100$) vs.\ the bottleneck (\textcolor{scbmred}{red}, $d_{\Z}=5$). Bottleneck conditioning yields a consistent finite-sample advantage, with both estimators converging for large $n_{\mathrm{train}}$. Results of additional parameter configurations provided in Appendix~\ref{app:ssec:adj_viz}.}
    \label{fig:bottleneck_adjustment}
\end{figure}

\subsubsection{Backdoor adjustment}\label{sssec:bd}
Let $\*C, \*X, \*Y \in \mathbb{R}^d$ be confounder, treatment, and
outcome satisfying the linear structural equations
\begin{align}
    \*C &= \boldsymbol{\eta}_C, \label{eq:bd_c} \\
    \*X &= \*C\,M_{CX} + \boldsymbol{\eta}_X, \label{eq:bd_x} \\
    \*Y &= \*X\,M_{XY} + \*C\,M_{CY} + \boldsymbol{\eta}_Y, \label{eq:bd_y}
\end{align}
where $M_{CX} = W^b_{CX}W^f_{CX}$ and $M_{CY} = W^b_{CY}W^f_{CY}$ with
$W^b_{CX}, W^b_{CY} \in \mathbb{R}^{d \times r}$ each of rank $r < d$, mutually independent noise terms,
and $\Sigma_{\eta_X}, \Sigma_{\eta_Y} \succ 0$.
The target is $M_{XY}$.
Define the bottleneck representations
$\*Z_{CX} := \*C\,W^b_{CX} \in \mathbb{R}^{n \times r}$ and
$\*Z_{CY} := \*C\,W^b_{CY} \in \mathbb{R}^{n \times r}$.

Unlike the mediator setting, the confounder $\*C$ has two outgoing bottlenecks: $\*Z_{CX}$ captures the part of $\*C$ that drives treatment, and $\*Z_{CY}$ captures the part that drives outcome. Both turn out to be valid adjustment variables for $M_{XY}$, though for different reasons.

\begin{lemma}[Bottleneck adjustment validity]
\label{lem:bd_validity}
Let $\*Z$ be an $r$-dimensional bottleneck representation of $\*C$.
The population OLS estimator of $\*Y$ on $(\*X, \*Z)$ consistently
identifies $M_{XY}$ if $\*Z$ satisfies either of the following conditions:
\begin{enumerate}
    \item[(a)] \textbf{Outcome sufficiency:}
          $\*C\,M_{CY}$ is a measurable function of $\*Z$, i.e.\ the
          confounding path $\*C \to \*Y$ factors through $\*Z$;
    \item[(b)] \textbf{Treatment sufficiency:}
          $\*X \perp \*C \mid \*Z$, i.e.\ $\*Z$ screens off $\*C$
          from $\*X$.
\end{enumerate}
In the linear SCBM, $\*Z_{CY}$ satisfies~\emph{(a)} and $\*Z_{CX}$
satisfies~\emph{(b)}.
\end{lemma}


Both bottleneck representations are valid, but they need not be equally efficient. Since $\*Z_{CY}$ directly explains the confounding path $\*C \to \*Y$, while $\*Z_{CX}$ only screens off $\*C$ from $\*X$ without fully capturing the outcome-side variation, we might expect $\*Z_{CY}$ to be the more efficient conditioning variable. The following proposition confirms this and establishes the complete efficiency ordering.

\begin{proposition}[Efficiency ordering of bottleneck adjustment]
\label{prop:bd_efficiency}
Let $\hat{M}_{XY}^{C}$, $\hat{M}_{XY}^{CX}$, and $\hat{M}_{XY}^{CY}$ denote
the OLS estimators of $M_{XY}$ obtained by conditioning on $\*C$,
$\*Z_{CX}$, and $\*Z_{CY}$, respectively.
All three are consistent.
Defining $\tilde{\*V}^S := \*V - \mathbb{E}[\*V \mid S]$ as the
population residual and $\delta_S := \tilde{\*Y}^S - \tilde{\*X}^S M_{XY}$
as the effective regression noise, the asymptotic variance of each estimator
is characterized by
\[
    \Phi_S \;:=\; Q_S^{-1}\,\Sigma_{\delta_S}\,Q_S^{-1},
    \qquad
    Q_S \;:=\; \operatorname{Cov}(\tilde{\*X}^S),
    \quad
    \Sigma_{\delta_S} \;:=\; \operatorname{Cov}(\delta_S),
\]
and satisfies $\Phi_{CY} \prec \Phi_{C} \prec \Phi_{CX}$, i.e.\
\[
    \operatorname{Avar}\!\left(\hat{M}_{XY}^{CY}\right)
    \;\prec\;
    \operatorname{Avar}\!\left(\hat{M}_{XY}^{C}\right)
    \;\prec\;
    \operatorname{Avar}\!\left(\hat{M}_{XY}^{CX}\right),
\]
with both strict inequalities holding whenever
$\ker(W^{b,\top}_{CY}) \cap \operatorname{col}(M_{CX}^\top) \neq \{0\}$
and
$\ker(W^{b,\top}_{CX}) \cap \operatorname{col}(W^b_{CY}) \neq \{0\}$,
respectively.
\end{proposition}


\Cref{fig:bottleneck_adjustment} (right) shows the gains of the outcome-side bottleneck over $C$.

The efficiency ordering $\Phi_{CY} \prec \Phi_C \prec \Phi_{CX}$ reveals an important asymmetry. $\*Z_{CX}$ plays the role of a \emph{propensity score} \citep{rosenbaum1983central}: it screens off $\*C$ from $\*X$ but leaves residual confounding in the outcome regression, mirroring the efficiency loss of propensity-score adjustment relative to conditioning on the full confounder \citep{hahn1998role, hirano2003efficient}. $\*Z_{CY}$, by directly explaining the path $\*C \to \*Y$, is both valid and strictly more efficient than conditioning on $\*C$ itself.

More generally, condition (a) of Lemma~\ref{lem:bd_validity} only requires the confounding path to factor through $\mathbf{Z}$, so any representation from which $\mathbf{Z}_{CY}$ can be recovered, including non-minimal ones (cf. \cref{ssec:misspec}), remains a valid adjustment variable. Such representations pay an efficiency cost relative to $\mathbf{Z}_{CY}$ (cf. proof of Proposition~\ref{prop:med-efficiency}), with $\mathbf{C}$ itself as the extreme case of Proposition~\ref{prop:bd_efficiency}. Finally, validity for a given task does not imply equivalence: $\mathbf{Z}_{CX}$ and $\mathbf{Z}_{CY}$ are both valid but generically capture different subspaces of $\mathbf{C}$. 

\section{Discussion}\label{sec:discussion}

We introduce structural causal bottleneck models, a novel family of graphical causal models for vector-valued variables, in which effects between high-dimensional variables are mediated by low-dimensional bottlenecks. We prove that bottlenecks are identifiable up to bijection from observational data, show that the bottleneck structure provides a principled remedy for three qualitatively different challenges in high-dimensional causal effect estimation, and validate these results experimentally. Our estimation method is relatively simple compared to most CRL and causal abstraction approaches, which we hope provides robustness to assumption violations in real data \citep{gamellabing2025sanity}.

\paragraph{Limitations.}

The primary limitation of our effect estimation results is the oracle assumption of access to the true bottleneck matrices $W^b_{ij}$. Two remedies arise naturally. First, \emph{transfer environments}---pilot RCTs, pre-intervention baselines, or do\-main shifts---can serve as held-out data from which bottlenecks are estimated before use in the target environment, as we demonstrate for the IV setting in Appendix~\ref{app:transfer_iv}. Second, pre-trained models may provide approximate bottlenecks without additional interventional data: a representation that contains the bottleneck remains a valid adjustment variable at an efficiency cost (\cref{sssec:bd}), whereas a predictive representation that misses part of it does not, and whether encoders trained for prediction achieve this containment is an open question.
Beyond the oracle assumption, our effect estimation theory is restricted to linear SCBMs and, like virtually all causal effect estimation methods, requires knowledge of the causal graph. Finally, the efficiency gains of bottleneck adjustment vanish under even small full-rank perturbations of the mechanisms (Appendix~\ref{app:ssec:misspec_lowrank}); if low rank cannot be justified, one should adjust for the full variables, which remain valid.

\paragraph{Outlook.}
This is an introductory work that establishes a novel modeling framework, leaving several directions for future work. A formal finite-sample characterization of the efficiency gains from using bottleneck conditioning sets would complement our asymptotic consistency results. Extending the IV non-identifiability result and its bottleneck-based remedy to nonlinear SCBMs is another open direction. A natural extension lies in developing application specific estimators, as our methodology generally permits fitting any model compatible with in- and output data modalities, without requiring a bespoke causal loss. Finally, investigating whether the bottleneck assumption can be exploited to assist causal discovery---partially or fully learning the graph we currently assume known---is a promising direction for future work.

\section*{Acknowledgements}

S.B. and J.R. received funding for this project from the European Research Council (ERC) Starting Grant CausalEarth under the European Union’s Horizon 2020 research and innovation program (Grant Agreement No. 948112). S.B. received support from the German Academic Scholarship Foundation. J.W. received support from the German Federal Ministry of Research, Technology, and Space (BMFTR) as part of the project MAC-MERLin (Grant Agreement No. 01IW24007) and from the European Regional Development Fund (ERDF) and the German Federal State of Saarland within the scope of (To)CERTAIN (Project ID EFRE-AuF-0000942). This work used resources of the Deutsches Klimarechenzentrum (DKRZ) granted by its Scientific Steering Committee (WLA) under project ID 1083.

\bibliography{references}
\bibliographystyle{plainnat}

\newpage
\appendix
\startcontents[appendix]

{\noindent\LARGE\bfseries Appendix}

\vspace{0.5\baselineskip}
\begingroup
\hypersetup{linkcolor=black}
\noindent\textbf{\large Table of Contents}

\noindent\rule{\linewidth}{0.4pt}
\printcontents[appendix]{}{1}[3]{}
\endgroup

\noindent\rule{\linewidth}{0.4pt}

\clearpage

\section{Proofs} \label{app:proofs}

\subsection{Proofs for \Cref{sec:id} (Identifiability of SCBMs)}\label{app:proofs:id}
\begin{replemma}{\ref{lem.ident1}}
    Let $\mathfrak{C} = \langle \cG, \cX, \cZ, \cH, \boldsymbol{\eta}, \mathcal{B},\mathcal{F}, \X \rangle$ be an SCBM. Assume that there are invertible maps $\psi_j: \cZ_j \to \cZ'_j$ for every endogeneous node $j \in \cV_{end}$ to some other spaces $\cZ'=(\cZ'_j)_{j \in \cV_{end}}$, and consider the functions $b'_j = \psi_j \circ b_j$ and $f'_j( \cdot \ ,\ \cdot ) = f_j(\psi_j^{-1}( \ \cdot \ ),\ \cdot)$. Then $\mathfrak{C}' = \langle \cG, \cX, \cZ', \boldsymbol{\eta}, \mathcal{B}',\mathcal{F}', \X \rangle$ with $\cF' = (f'_j)_j$ and  $\cB' = (b'_j)_j$ is structurally equivalent to $\mathfrak{C}$.
\end{replemma}

\begin{proof}[Proof of Lemma~\ref{lem.ident1}]
 The result follows directly from the fact that the structural equations in both models coincide:
 \begin{align*}
      m'_j\left((\X_i)_{i \in \mathrm{pa}(j)},\boldsymbol{\eta}_j\right) 
      &= f'_j \left( b_{j}'\left((\X_i)_{i \in \mathrm{pa}(j)} \right), \boldsymbol{\eta}_j \right) 
       \\ &= f_j(\psi_j^{-1}( \psi_j \left(  b_{j}\left((\X_i)_{i \in \mathrm{pa}(j)} \right) \right) \cdot \ ),\eta_j) 
       \\ &= f_j \left( b_{j}\left((\X_i)_{i \in \mathrm{pa}(j)} \right), \boldsymbol{\eta}_j \right) 
       \\ &= m_j\left((\X_i)_{i \in \mathrm{pa}(j)},\boldsymbol{\eta}_j\right) 
 \end{align*}
\end{proof}

\begin{replemma}{\ref{lem.converse_direction}}
    Let $\mathfrak{C} = \langle \cG, \cX, \cZ, \cH, \boldsymbol{\eta}, \mathcal{B},\mathcal{F}, \X \rangle$ and $\mathfrak{C}' = \langle \cG, \cX, \cZ', \cH', \boldsymbol{\eta}, \mathcal{B}',\mathcal{F}', \X \rangle$  be two SCBMs with additive noises, i.e. $f_j(\mathbf{z}_{\mathrm{pa}(j)},\eta_j) = \tilde{f}_j(\mathbf{z}_{\mathrm{pa}(j)})+ \eta_j$, and similarly for $f'_j$. Assume that the functions $\tilde{f}_{j}: \cZ_{j} \to \cX_j, \ \tilde{f}'_{j}: \cZ_{j}' \to \cX_j$ are almost surely injective. If $\mathfrak{C}$ and $\mathfrak{C}'$ are structurally equivalent, there is an invertible function $\psi_j: \cZ_j \to \cZ_j'$ such that $b' = \psi \circ b$ and $f_j' = f_j \circ \psi^{-1}$.
\end{replemma}

\begin{proof}[Proof of Lemma~\ref{lem.converse_direction}]
    Since the models are structurally equivalent and the noise terms are the same in both models, it follows directly that $f_j \circ b_j = f_j' \circ b_j'$ $P_{\X}$-almost surely. In particular the maps $f_j, f_j'$ have the same range, so that the map $\psi_j = f_j'^{-1} \circ F_j$ is well-defined and has the desired properties.   
\end{proof}

\subsection{Proofs for \Cref{sec:est} (Effect estimation in SCBMs)}\label{app:est_proofs}
\subsubsection{Proofs for \Cref{ssec:iv} (Unobserved confounding)}

\begin{repproposition}{\ref{prop:2sls_nonid}}
Let $\*I, \*X, \*Y, \*U \in \mathbb{R}^{d}$ be the instrument, treatment, outcome,
and unobserved confounder, respectively, related by the linear structural equations
\begin{align}
    \*I &= \boldsymbol{\eta}_I,\\
    \*U &= \boldsymbol{\eta}_U,\\
    \*X &= \*I\, M_{IX} + \*U\, M_{UX} + \boldsymbol{\eta}_X,\\
    \*Y &= \*X\, M_{XY} + \*U\, M_{UY} + \boldsymbol{\eta}_Y,
\end{align}
where the mechanism matrix $M_{IX} \in \mathbb{R}^{d \times d}$ admits the SCBM
factorization
\[
    M_{IX} \;=\; W^{b}_{IX}\,W^{f}_{IX},
    \qquad W^{b}_{IX} \in \mathbb{R}^{d \times r},\quad
    W^{f}_{IX} \in \mathbb{R}^{r \times d},
\]
with $W^{b}_{IX}$ and $W^{f}_{IX}$ each of rank $r$, so that
$\operatorname{rank}(M_{IX}) = r < d$.
Assume $\*I \perp (\*U,\,\*\eta_X,\,\*\eta_Y)$ and
$\Sigma_I := \mathbb{E}[\*I^{\!\top}\*I] \succ 0$.
Then the population 2SLS estimator for $M_{XY}$ has no unique solution.
\end{repproposition}

\begin{proof}[Proof of Proposition~\ref{prop:2sls_nonid}]
\textbf{First stage.}
The population OLS of $\*X$ on $\*I$ solves
$\widehat{M}_{\mathrm{FS}} = \operatorname{arg\,min}_{M}\,
\mathbb{E}[\|\*X - \*I M\|^{2}]$.
Substituting \eqref{eq:scm_x} and using the exogeneity conditions
$\*I \perp \*U$, $\*I \perp \boldsymbol{\eta}_X$ gives
\[
    \mathbb{E}[\*I^{\!\top}\*X] \;=\; \Sigma_I\,M_{IX},
\]
so $\widehat{M}_{\mathrm{FS}} = \Sigma_I^{-1}\Sigma_I M_{IX} = M_{IX}$.
The population first-stage fitted values are therefore
$\hat{\*X} = \*I\,M_{IX}$.

\textbf{Second stage.}
2SLS regresses $\*Y$ on $\hat{\*X}$, leading to the population normal equations
\begin{equation}
    G\,\Gamma \;=\; M_{IX}^{\top}\,\mathbb{E}[\*I^{\!\top}\*Y],
    \qquad G \;:=\; M_{IX}^{\top}\,\Sigma_I\,M_{IX}.
    \label{eq:normal}
\end{equation}
Using exogeneity ($\*I \perp \*U$, $\*I \perp \boldsymbol{\eta}_Y$) and
\eqref{eq:scm_y}, we obtain
$\mathbb{E}[\*I^{\!\top}\*Y] = \Sigma_I\,M_{IX}\,M_{XY}$,
so the right-hand side of \eqref{eq:normal} equals $G\,M_{XY}$ and the
normal equations reduce to
\begin{equation}
    G\,\Gamma \;=\; G\,M_{XY}.
    \label{eq:normal_reduced}
\end{equation}
Since $\Sigma_I \succ 0$ and $\operatorname{rank}(M_{IX}) = r < d$,
we have $\operatorname{rank}(G) = r < d$, so $G$ is singular with
$\ker(G) = \ker(M_{IX})$, a subspace of dimension $d - r > 0$.
Consequently \eqref{eq:normal_reduced} is satisfied by
$\Gamma = M_{XY} + N$ for any $N \in \ker(M_{IX})$, yielding infinitely
many solutions. The 2SLS estimator for $M_{XY}$ is therefore not
uniquely defined.
\end{proof}

\begin{reptheorem}{\ref{thm:oracle_2sls}}
Under the setup of Proposition~\ref{prop:2sls_nonid}, additionally let
$M_{XY} = W^b_{XY}W^f_{XY}$ with $W^b_{XY} \in \mathbb{R}^{d \times r}$,
$W^f_{XY} \in \mathbb{R}^{r \times d}$, each of rank $r$.
Suppose the rank condition
\begin{equation}
    \operatorname{rank}\!\left(W^f_{IX}\,W^b_{XY}\right) = r
    \tag{RC}
\end{equation}
holds. Given oracle access to $W^b_{IX}$ and $W^b_{XY}$, define the
$r$-dimensional bottleneck representations
\[
    \*Z_{IX} \;:=\; \*I\,W^b_{IX},
    \qquad
    \*Z_{XY} \;:=\; \*X\,W^b_{XY}.
\]
The oracle bottleneck 2SLS procedure regresses $\*Z_{XY}$ on $\*Z_{IX}$
(first stage) to obtain fitted values $\hat{\*Z}_{XY}$, then regresses
$\*Y$ on $\hat{\*Z}_{XY}$ (second stage) to obtain $\hat{\Phi}$, and
forms the estimator
\[
    \widehat{M}_{XY} \;:=\; W^b_{XY}\,\hat{\Phi}.
\]
Then $\widehat{M}_{XY} = M_{XY}$.
\end{reptheorem}

\begin{proof}[Proof of Theorem~\ref{thm:oracle_2sls}]
Let $A := W^f_{IX}\,W^b_{XY} \in \mathbb{R}^{r \times r}$ and
$\Sigma_{b_{IX}} := (W^b_{IX})^{\!\top}\Sigma_I\,W^b_{IX} \in \mathbb{R}^{r \times r}$.

\textbf{First stage.}
The population Gram matrix for the OLS of $\*Z_{XY}$ on $\*Z_{IX}$ is
$\Sigma_{Z_{IX}}$.
Since $\Sigma_I \succ 0$ and $W^b_{IX}$ has rank $r$, we have
$\Sigma_{Z_{IX}} \succ 0$, so the first stage has a unique solution.
Using exogeneity ($\*I \perp \*U$, $\*I \perp \boldsymbol{\eta}_X$) and \eqref{eq:scm_x}:
\[
    \mathbb{E}[\*Z_{IX}^{\!\top}\*Z_{XY}]
    \;=\; (W^b_{IX})^{\!\top}\mathbb{E}[\*I^{\!\top}\*X]\,W^b_{XY}
    \;=\; (W^b_{IX})^{\!\top}\Sigma_I\,M_{IX}\,W^b_{XY}
    \;=\; \Sigma_{Z_{IX}}\,A,
\]
so the population first-stage coefficient is
$A^* = \Sigma_{Z_{IX}}^{-1}\Sigma_{Z_{IX}}\,A = A$,
and the population fitted values are
$\hat{\*Z}_{XY} = \*Z_{IX}\,A = \*I\,M_{IX}\,W^b_{XY}$.

\textbf{Second stage.}
The population Gram matrix for the OLS of $\*Y$ on $\hat{\*Z}_{XY}$ is
\[
    H \;:=\; \mathbb{E}[\hat{\*Z}_{XY}^{\!\top}\hat{\*Z}_{XY}]
    \;=\; (W^b_{XY})^{\!\top}M_{IX}^{\!\top}\Sigma_I\,M_{IX}\,W^b_{XY}
    \;=\; A^{\!\top}\Sigma_{Z_{IX}}\,A.
\]
Condition~\eqref{eq:rc} gives $\operatorname{rank}(A) = r$, so
$H \succ 0$ and the second stage has a unique solution.
Using exogeneity ($\*I \perp \*U$, $\*I \perp \boldsymbol{\eta}_Y$) and \eqref{eq:scm_y},
the same calculation as in the proof of Proposition~\ref{prop:2sls_nonid}
gives $\mathbb{E}[\*I^{\!\top}\*Y] = \Sigma_I\,M_{IX}\,M_{XY}$, and hence
\[
    \mathbb{E}[\hat{\*Z}_{XY}^{\!\top}\*Y]
    \;=\; (W^b_{XY})^{\!\top}M_{IX}^{\!\top}\Sigma_I\,M_{IX}\,M_{XY}
    \;=\; (W^b_{XY})^{\!\top}M_{IX}^{\!\top}\Sigma_I\,M_{IX}\,W^b_{XY}\,W^f_{XY}
    \;=\; H\,W^f_{XY}.
\]
Since $H \succ 0$, the normal equations yield uniquely
$\hat{\Phi} = H^{-1}H\,W^f_{XY} = W^f_{XY}$.

\textbf{Composition.}
\[
    \widehat{M}_{XY} \;=\; W^b_{XY}\,\hat{\Phi}
    \;=\; W^b_{XY}\,W^f_{XY} \;=\; M_{XY}. \qquad\qedhere
\]
\end{proof}

\subsubsection{Proofs for \Cref{sssec:med} (Mediator adjustment)}

\begin{replemma}{\ref{lem:med-validity}}
Let $\hat{M}_{XY}^{S}$ be the OLS estimator of $M_{XY}$ from the regression of $\*Y$
on $(\*X,\, S)$.
Both $S = \*M$ and $S = \*Z_{MY}$ yield consistent estimators:
\[
  \hat{M}_{XY}^{M} \xrightarrow{p} M_{XY}, \qquad
  \hat{M}_{XY}^{Z_{MY}} \xrightarrow{p} M_{XY}.
\]
\end{replemma}

\begin{proof}[Proof of Lemma~\ref{lem:med-validity}]
Substituting $M_{MY} = W^b_{MY}W^f_{MY}$ into \eqref{eq:med-y} gives
\[
  \*Y = \*X M_{XY} + \*Z_{MY} W^f_{MY} + \boldsymbol{\eta}_Y.
\]
For $S = \*Z_{MY}$, this is a valid linear regression:
the regressor matrix $(\*X,\, \*Z_{MY})$ is uncorrelated with $\boldsymbol{\eta}_Y$
because $\*Z_{MY} = \*M W^b_{MY}$ depends only on $(\*X, \boldsymbol{\eta}_M)$,
and all noise terms are mutually independent.
OLS is therefore consistent for $M_{XY}$.
The case $S = \*M$ follows identically,
since $\*M$ also determines $\*Z_{MY}$ and contains strictly more information.
\end{proof}

\begin{repproposition}{\ref{prop:med-efficiency}}
Under the linear SCBM, the asymptotic covariance of $\hat{M}_{XY}^{S}$
is $\Phi_S := Q_S^{-1}\Sigma_{\boldsymbol{\eta}_Y}Q_S^{-1}$, where
$Q_S := \operatorname{Cov}(\tilde{\*X}^S)$ denotes the covariance of
the Frisch--Waugh--Lovell \citep{frisch1933partial,lovell1963seasonal} (FWL) residual of $\*X$ after projecting out $S$.  Then
\[
  \Phi_{Z_{MY}} \preceq \Phi_{M}.
\]
\end{repproposition}

\begin{proof}[Proof of Proposition~\ref{prop:med-efficiency}]
By the Frisch--Waugh--Lovell theorem, $\hat{M}_{XY}^{S}$ equals OLS of
$\tilde{\*Y}^S$ on $\tilde{\*X}^S$, with residual
$\*\delta_S := \tilde{\*Y}^S - \tilde{\*X}^S M_{XY}$.
Substituting the structural equation and using the independence
$\boldsymbol{\eta}_Y \perp (\*X, \*M)$, one obtains
$\*\delta_S \xrightarrow{d} \boldsymbol{\eta}_Y$ for both $S\in\{\*M,\*Z_{MY}\}$,
so $\Sigma_{\delta_S} = \Sigma_{\boldsymbol{\eta}_Y}$ in both cases.

Since $\*Z_{MY} = \*M W^b_{MY}$ is a function of $\*M$,
the $\sigma$-algebra generated by $\*Z_{MY}$ is coarser than that of $\*M$.
By the law of total variance, conditioning on the finer variable removes more variance:
\[
  Q_{M}
  = \operatorname{Cov}\!\bigl(\*X - \mathbb{E}[\*X \mid \*M]\bigr)
  \preceq
  \operatorname{Cov}\!\bigl(\*X - \mathbb{E}[\*X \mid \*Z_{MY}]\bigr)
  = Q_{Z_{MY}}.
\]
Hence $Q_{Z_{MY}}^{-1} \preceq Q_{M}^{-1}$, and since
$\Phi_S = Q_S^{-1}\Sigma_{\boldsymbol{\eta}_Y}Q_S^{-1}$:
\[
  \Phi_{Z_{MY}} = Q_{Z_{MY}}^{-1}\Sigma_{\boldsymbol{\eta}_Y}Q_{Z_{MY}}^{-1}
  \preceq
  Q_{M}^{-1}\Sigma_{\boldsymbol{\eta}_Y}Q_{M}^{-1} = \Phi_{M}. \qquad\square
\]
\end{proof}

\subsubsection{Proofs for \Cref{sssec:bd} (Backdoor adjustment)}

\begin{replemma}{\ref{lem:bd_validity}}
Let $\*Z$ be an $r$-dimensional bottleneck representation of $\*C$.
The population OLS estimator of $\*Y$ on $(\*X, \*Z)$ consistently
identifies $M_{XY}$ if $\*Z$ satisfies either of the following conditions:
\begin{enumerate}
    \item[(a)] \textbf{Outcome sufficiency:}
          $\*C\,M_{CY}$ is a measurable function of $\*Z$, i.e.\ the
          confounding path $\*C \to \*Y$ factors through $\*Z$;
    \item[(b)] \textbf{Treatment sufficiency:}
          $\*X \perp \*C \mid \*Z$, i.e.\ $\*Z$ screens off $\*C$
          from $\*X$.
\end{enumerate}
In the linear SCBM, $\*Z_{CY}$ satisfies~\emph{(a)} and $\*Z_{CX}$
satisfies~\emph{(b)}.
\end{replemma}

\begin{proof}[Proof of Lemma~\ref{lem:bd_validity}]
Identification requires $\mathbb{E}[\*Y \mid \*X, \*Z] =
\*X\,M_{XY} + g(\*Z)$ for some function $g$.
Substituting~\eqref{eq:bd_y}:
\[
    \mathbb{E}[\*Y \mid \*X, \*Z]
    \;=\; \*X\,M_{XY}
    \;+\; \mathbb{E}[\*C\,M_{CY} \mid \*X, \*Z].
\]
It suffices to show the second term depends only on $\*Z$.

\textbf{Case~(a), $\*Z = \*Z_{CY}$.}
Since $M_{CY} = W^b_{CY}W^f_{CY}$, we have
$\*C\,M_{CY} = \*Z_{CY}\,W^f_{CY}$,
so $\mathbb{E}[\*C\,M_{CY} \mid \*X, \*Z_{CY}] = \*Z_{CY}\,W^f_{CY}$
is a function of $\*Z_{CY}$ alone.
Hence $\mathbb{E}[\*Y \mid \*X, \*Z_{CY}]
= \*X\,M_{XY} + \*Z_{CY}\,W^f_{CY}$.

\textbf{Case~(b), $\*Z = \*Z_{CX}$.}
From~\eqref{eq:bd_x}, $\*X = \*Z_{CX}\,W^f_{CX} + \boldsymbol{\eta}_X$ with
$\boldsymbol{\eta}_X \perp \*C$, so the variation in $\*X$ given $\*Z_{CX}$
is $\boldsymbol{\eta}_X$, which is independent of $\*C$.
Therefore $\*X \perp \*C \mid \*Z_{CX}$, giving
$\mathbb{E}[\*C \mid \*X, \*Z_{CX}] = \mathbb{E}[\*C \mid \*Z_{CX}]
=: g(\*Z_{CX})$.
Hence $\mathbb{E}[\*Y \mid \*X, \*Z_{CX}]
= \*X\,M_{XY} + g(\*Z_{CX})\,M_{CY}$.

In both cases the coefficient on $\*X$ in the population OLS is $M_{XY}$.
\end{proof}

\begin{repproposition}{\ref{prop:bd_efficiency}}
Let $\hat{M}_{XY}^{C}$, $\hat{M}_{XY}^{CX}$, and $\hat{M}_{XY}^{CY}$ denote
the OLS estimators of $M_{XY}$ obtained by conditioning on $\*C$,
$\*Z_{CX}$, and $\*Z_{CY}$, respectively.
All three are consistent.
Defining $\tilde{\*V}^S := \*V - \mathbb{E}[\*V \mid S]$ as the
population residual and $\delta_S := \tilde{\*Y}^S - \tilde{\*X}^S M_{XY}$
as the effective regression noise, the asymptotic variance of each estimator
is characterized by
\[
    \Phi_S \;:=\; Q_S^{-1}\,\Sigma_{\delta_S}\,Q_S^{-1},
    \qquad
    Q_S \;:=\; \operatorname{Cov}(\tilde{\*X}^S),
    \quad
    \Sigma_{\delta_S} \;:=\; \operatorname{Cov}(\delta_S),
\]
and satisfies $\Phi_{CY} \prec \Phi_{C} \prec \Phi_{CX}$, i.e.\
\[
    \operatorname{Avar}\!\left(\hat{M}_{XY}^{CY}\right)
    \;\prec\;
    \operatorname{Avar}\!\left(\hat{M}_{XY}^{C}\right)
    \;\prec\;
    \operatorname{Avar}\!\left(\hat{M}_{XY}^{CX}\right),
\]
with both strict inequalities holding whenever
$\ker(W^{b,\top}_{CY}) \cap \operatorname{col}(M_{CX}^\top) \neq \{0\}$
and
$\ker(W^{b,\top}_{CX}) \cap \operatorname{col}(W^b_{CY}) \neq \{0\}$,
respectively.
\end{repproposition}

\begin{proof}[Proof of Proposition~\ref{prop:bd_efficiency}]
We compute $Q_S$ and $\Sigma_{\delta_S}$ for each conditioning variable.

\textbf{Conditioning on $\*C$.}
From~\eqref{eq:bd_x}--\eqref{eq:bd_y} and $\boldsymbol{\eta}_X \perp \*C$:
\[
    \tilde{\*X}^{C} = \boldsymbol{\eta}_X,
    \qquad
    \tilde{\*Y}^{C} = \tilde{\*X}^{C}\,M_{XY} + \boldsymbol{\eta}_Y,
    \qquad
    \delta_{C} = \boldsymbol{\eta}_Y.
\]
Hence $Q_{C} = \Sigma_{\boldsymbol{\eta}_X}$ and $\Sigma_{\delta_{C}} = \Sigma_{\boldsymbol{\eta}_Y}$.

\textbf{Conditioning on $\*Z_{CY}$.}
Since $\*C\,M_{CY} = \*Z_{CY}\,W^f_{CY}$ is a function of $\*Z_{CY}$:
\[
    \tilde{\*Y}^{CY} \;=\; \tilde{\*X}^{CY}\,M_{XY} + \boldsymbol{\eta}_Y,
    \qquad
    \delta_{CY} = \boldsymbol{\eta}_Y,
    \qquad
    \Sigma_{\delta_{CY}} = \Sigma_{\boldsymbol{\eta}_Y}.
\]
The treatment residual retains the component of $\*C$ outside
$\operatorname{col}(W^b_{CY})$:
\[
    Q_{CY} \;=\; \operatorname{Cov}(\tilde{\*X}^{CY})
    \;=\; M_{CX}^\top\,\operatorname{Cov}(\*C \mid \*Z_{CY})\,M_{CX}
    + \Sigma_{\boldsymbol{\eta}_X}
    \;\succeq\; \Sigma_{\boldsymbol{\eta}_X} \;=\; Q_{C}.
\]
Since $\Sigma_{\delta_{CY}} = \Sigma_{\delta_{C}} = \Sigma_{\boldsymbol{\eta}_Y}$
but $Q_{CY} \succeq Q_{C}$, we have $Q_{CY}^{-1} \preceq Q_{C}^{-1}$
and therefore $\Phi_{CY} \preceq \Phi_{C}$.

\textbf{Conditioning on $\*Z_{CX}$.}
Since $\*X = \*Z_{CX}\,W^f_{CX} + \boldsymbol{\eta}_X$ with $\boldsymbol{\eta}_X \perp \*C$:
\[
    \tilde{\*X}^{CX} = \boldsymbol{\eta}_X,
    \qquad
    Q_{CX} = \Sigma_{\boldsymbol{\eta}_X} = Q_{C}.
\]
However, $\*Z_{CX}$ does not fully explain $\*C\,M_{CY}$, so the
effective noise retains residual confounding:
\[
    \delta_{CX}
    \;=\; \bigl(\*C - \mathbb{E}[\*C \mid \*Z_{CX}]\bigr)\,M_{CY}
    + \boldsymbol{\eta}_Y,
\]
giving
\[
    \Sigma_{\delta_{CX}}
    \;=\; M_{CY}^\top\,\operatorname{Cov}(\*C \mid \*Z_{CX})\,M_{CY}
    + \Sigma_{\boldsymbol{\eta}_Y}
    \;\succeq\; \Sigma_{\boldsymbol{\eta}_Y} \;=\; \Sigma_{\delta_{C}}.
\]
Since $Q_{CX} = Q_{C}$ but $\Sigma_{\delta_{CX}} \succeq \Sigma_{\delta_{C}}$,
we have $\Phi_{CX} \succeq \Phi_{C}$.

\textbf{Strict inequalities.}
$\Phi_{CY} \prec \Phi_{C}$ iff $Q_{CY} \succ Q_{C}$, which holds
iff $\operatorname{Cov}(\*C \mid \*Z_{CY})$ is nonzero in the direction
of $M_{CX}$, i.e.\ $\ker(W^{b,\top}_{CY}) \cap \operatorname{col}(M_{CX}^\top)
\neq \{0\}$.
$\Phi_{CX} \succ \Phi_{C}$ iff $\Sigma_{\delta_{CX}} \succ \Sigma_{\boldsymbol{\eta}_Y}$,
which holds iff $\operatorname{Cov}(\*C \mid \*Z_{CX})$ is nonzero in the
direction of $M_{CY}$, i.e.\ $\ker(W^{b,\top}_{CX}) \cap
\operatorname{col}(W^b_{CY}) \neq \{0\}$.
Both conditions hold with probability one whenever the weight matrices
are drawn from any distribution absolutely continuous with respect to
Lebesgue measure.
\end{proof}

\section{Special cases and interpretations}\label{sec:spec}
The class of SCBMs is large and allows to model interactions among spatial as well as temporal phenomena as we illustrate below. We also show that SCBMs can be interpreted through the lens of Tishby's information bottlenecks \citep{tishby2000information}.

\paragraph{Modeling interactions between random fields.}

The quantity of interest in geo-spatial modelling are often represented by random fields over some spatial area $\cD$. This spatial area can be a discrete grid or a continuous domain. The random field is a random function over $\cD$, which in the discrete case can be identified with a vector of dimension $N$, where $N$ is the number of grid points. The spatial structure in such a model is reflected in the correlation structure of the field. For instance, if $\cD$ is a discrete grid, a typical assumption is the spatial Markov property, which states that every node $X_i, \ i \in \cD$, is independent of non-adjacent nodes given its neighbors. In continuous random fields, a standard assumption would be that the correlation between two nodes $X(y), X(z), \ y,z \in \cD$, decreases when the spatial distance of $y$ and $z$ increases. An SCBM could for example consist of three random fields of noises $\boldsymbol{\eta}_1, \boldsymbol{\eta}_2, \boldsymbol{\eta}_3 $ over bounded spatial domains $\cD_1, \cD_2, \cD_3$ where fields affect each other only through weighted spatial means, e.g.

\begin{align*}
    \X_1 &:= \boldsymbol{\eta}_1 \\
    \X_2 &:= A\cdot \int_{\cD_1} \alpha(y) \X_1(y) \ dy +  \boldsymbol{\eta}_2 \\
    \X_3 &:= B \cdot \int_{\cD_1} \beta(y) \X_1(y) \ dy + C \cdot \int_{\cD_2} \gamma(z) \X_2(z) \ dz +  \boldsymbol{\eta}_3,
\end{align*}

where $\alpha,\beta,\gamma$ are normalized to one, e.g. $\int_{\cD_1} \alpha(y)  \ dy = 1.$
In this case, the spaces $\cX_i$ are function spaces over the respective domains, e.g. $\cX_i = L^2(\cD_i)$, the bottleneck spaces are one dimensional $\cZ_i \cong \mathbb R$ and the bottleneck functions are given by the weighted integrals. The effect functions are simple embeddings by multiplication with a constant, e.g. $f_{12}: \cZ_{12} \to L^2(\cD_2): z \mapsto A z \cdot  \mathbbm{1}_{\cD_2}$, where $\mathbbm{1}_{\cD_2}$ denotes the constant function on $\cD_2$.

\paragraph{Temporal processes in the frequency domain.}

SCBMs can also model interactions of temporal processes in the frequency domain. For instance, $\X_1$ and $\X_2$ may be continuous stochastic processes viewed as random variables with values in $L^2(\mathbb R_{\geq 0})$. Their child process $\Y$ may be affected by $\X_1,\X_2$ within at most $K$ frequency components with frequencies $\omega_1,\dots.\omega_K$, see \citep{schur2026decor} for an example of this. Thus, the bottleneck maps $b_i: L^2(\mathbb R_{\geq 0}) \to \mathbb R^K$ map the processes to the $K$ Fourier coefficients at $\omega_1,\dots.\omega_K$, and the effect function expresses $\Y$ as $\Y_t = \sum_{k=1,\dots_K} (\lambda_1 b_1(\x_1)+ \lambda_2 b_2(\x_2)) e^{-it \omega_k} + \boldsymbol{\eta}$ where $\boldsymbol{\eta}$ is a noise process, for instance a Brownian motion.

\paragraph{SCBMs as information bottlenecks.}

Consider two high-dimensional multivariate random variables $\X,\Y$. The idea of the information bottleneck framework of \citet{tishby2000information} is to find a minimal sufficient statistic for $\T = g(\X)$ that captures as much information about the target variable $\Y$ as possible.  This goal is captured by the \emph{minimal compression} optimization objective
\begin{align*}
    \min_{\T \ : \ I(\X,\Y | \T) = 0 } I(\X,\T).
\end{align*}
The \emph{independence constraint} $I(\X,\Y | \T) = 0$ is equivalent to $I(\X,\Y) = I(\T,\Y)$. Note also that the data processing inequality enforces
\begin{align*}
    I(g(\X),\Y) \leq  I(\X,\Y)
\end{align*}
for any abstraction function $g$, an information bottleneck is thus nothing but an abstraction of $\X$ that maximizes $I(g(\X),\Y)$. This property is used to incorporate the independence constraint as a soft constraint in the joint objective
\begin{align*}
    \min_{\T} \left(I(\X,\T) - \beta I(\Y,\T) \right)
\end{align*}
with regularization parameter $\beta > 0$.
Consider now an SCBM with a shared bottleneck for all children of a node, i.e., $\Z_{ij} =\Z_{i}$ for all $j \in \ch(i)$, over variables $\X_i$ with bottleneck variables $\Z_i = b_i(\X_i)$ for every non-sink node $i$, i.e.,
\begin{align*}
    \X_j = f_j(\Z_{i_1},\dots, \Z_{i_m})+ \boldsymbol{\eta}_j \qquad \pa(j) = \{i_1,\dots,i_m \}. 
\end{align*}
Our goal is to produce an optimization objective and constraints to learn the bottleneck variables $\Z_i$. The minimal compression requirement now means that $\Z_i$ contains the minimal necessary information about $\X$ once $\X$'s  parents are known, i.e. $ \Z_i \in \mathrm{argmin} \ I(\X_i,\Z_i| \Z_{\pa(i)})$ for all $i$.

The bottleneck $\Z_i$ is intended to be maximally informative about the children of $\X_i$ provided that all backdoor paths to these children are closed which can be achieved by conditioning on the parental bottlenecks $ \Z_{\pa(i)} = \{ \Z_k, k \in \pa(i) \}$. Thus, the conditional independence constraint is
\begin{align*}
    I(\X_{\ch(i)}, \X_i | \Z_i, \Z_{\pa(i)}) = 0.
\end{align*}
for non-sink nodes $i$. By the chain rule for conditional mutual information, this is equivalent to
\begin{align*}
     I(\X_{\ch(i)}, \Z_i | \Z_{\pa(i)}) =  I(\X_{\ch(i)}, \X_i | \Z_{\pa(i)}).
\end{align*}
Thus, writing $\Z = (\Z_i)_{i \in \cV_{ns}}$, we want to solve the family of optimization objectives
\begin{align*}
    \min_{\Z \ : \ I(\X_{\ch(i)}, \X_i | \Z_i, \Z_{\pa(i)}) = 0 } I(\X_i,\Z_i|\Z_{\pa(i)}), \quad i \in \cV_{ns}.
\end{align*}
These objectives are linked by the fact that the i-th objective involves not only the bottleneck $\Z_i$ but also the bottlenecks of the parent variables which may be part of other objectives as well. Instead of solving for $\Z$ globally, it is possible to consider these objectives sequentially along a causal order. To formalize this, define the \emph{causal grading} $\cV = \bigsqcup_s \cV_s$ where
$\cV_0 = \cV_{ex}$ and $i \in \cV_s$ for $s>0$ if and only if $i \notin \cV_q, q <s$ and $\pa(i) \subset \bigsqcup_{q<s} \cV_q $.
Fixing a causal order $(i_1,\dots,i_n)$ that respects the causal grading allows us to solve the objectives 
\begin{align*}
    \min_{\Z_i}   \left( I(\X_i,\Z_i|\Z_{\pa(i)}) - \beta_i  I(\X_{\ch(i)}, \Z_i | \Z_{\pa(i)}) \right)
\end{align*}
sequentially along this order since all parental bottlenecks appearing in the $i$-th equation have been computed as solutions of earlier objectives, see also Appendix~\ref{app:ssec:estim}.

\section{Detailed related work}\label{sec:rel_work}

\paragraph{Causal representation learning.}
Although we also learn representations of high-dimensional observations, our approach differs from canonical causal representation learning (CRL) approaches \citep{scholkopf2021causal, brehmer2022weakly, ahuja2023interventional, lippe2022citris, lippe2022icitris, lippe2023biscuit, squires2023linear, buchholz2023learning, liang2023causal, varici2025scorebased, bing2024identifying, lachapelle2026nonparametric, vonkugelgen2024nonparametric, yao2024multi} in key ways. Instead of mapping observations to SCM nodes, we learn \emph{multiple} maps tied to mechanisms between known (vector-valued) SCM nodes. Unlike most CRL methods that seek latent variables up to permutation and rescaling, we allow a broader class of invertible transformations for identifiability. Additionally, while CRL assumes invertible maps, we focus on surjective maps to enable true dimension reduction by discarding irrelevant observational details.

CRL aims to recover a latent low-dimensional SCM including its constituting variables. In contrast, learning bottlenecks in an SCBM is focused on causal effect estimation. We learn representations assuming a known graph and target a specific downstream causal query. Defining causal representations by their usefulness for downstream tasks, rather than recovering a single postulated ground-truth latent model, has been proposed as a step forward for CRL \citep{jorgensen2025causal}, and our targeted approach aligns with this view.

\paragraph{Causal abstractions.}
Our work is related to causal abstraction learning, but \emph{what} we abstract is different from existing approaches. Common approaches abstract causal models as a whole \citep{zennaro2023jointly, felekis2024causal, xia2024neural, massidda2024learning, d2025causal, madaleno2026coarsening}, while our notion of abstraction is a within-model operation that reduces a random vector to the essentials needed for downstream effect estimations. 

When a bottleneck is shared across all children of a node and discrete, the inverse image of the bottleneck function recovers exactly the intra-variable partitions of \citet{xia2024neural}. A factored discrete bottleneck can then be understood as a target-specific variant of their setting. More precisely, the level sets of each bottleneck function $b_{ij}$ define which states of $\X_i$ are causally equivalent for child $\X_j$---a per-child instance of the structural invariances of \citet{kroeger2025learning} (Definition~7). In the factored SCBM setting, distinct children carry distinct invariances, which generalizes beyond what a single structural invariance---required by Definition~7 to hold for all children simultaneously---can express. Two distinctions persist across all cases: SCBMs require no causal model over the micro-level variables---only a joint distribution, e.g.\ a Markov random field---and bottleneck spaces need not be discrete.

An approach closely related to ours is the causal feature learning method proposed by \citet{chalupka2017causal}, which learns to partition the spaces of a cause-effect pair to extract high-level features from low-level observations. The SCBM framework extends this idea to settings with more than two variables and allows for continuous high-level variables (bottlenecks), as opposed to only permitting discrete variables as a result of a clustering operation.

The Causal Information Bottleneck (CIB) \citep{simoes2025causal} seeks a representation for a specific causal query, but differs from our approach in its problem setting and method. The CIB framework focuses on a single cause-effect pair which must be identifiable via the backdoor criterion, while we handle arbitrary acyclic graphs. \citet{simoes2025causal} extend the Information Bottleneck Lagrangian of \citet{tishby2000information} to trade off compression and what they call ``causal control" via a constrained optimization, whereas we assume a data generating process that conceptualizes optimal low-dimensional representations directly and enables bottleneck estimation through regression or likelihood-based losses. The CIB method requires access to all conditional probability distributions, as it must compute $p(\Y \mid \text{do}(\X = \x) )$, while we use only observational data and graph structure. Additionally, CIB assumes discrete variables; we allow both discrete and continuous.

\paragraph{Dimension reduction.}
Principal component analysis (PCA) \citep{pearson1901liii} and similar dimension reduction techniques can be understood to represent the opposing side of a trade-off between training effort required before application and guarantees over retained relevant information for a downstream task. PCA does not need to be fit to a specific setting, but may also discard data relevant to a downstream task. In our framework, we do require additional data to fit an estimator of a given bottleneck, but we gain guarantees of the compressed representation being optimal for describing the specific mechanism of interest.

In a setting with only one regressor $\X$ and one regressand $\Y$, the task of finding an appropriate bottleneck variable is known as sufficient dimension(ality) reduction \citep{globerson2003sufficient, li2007sparse,li2018sufficient}. If in addition the relationship of $\X$ and $\Y$ is assumed linear, \emph{reduced-rank regression} \citep{izenman1975reduced} implicitly assumes the existence of a low-dimensional bottleneck variable by restricting the rank of the effect matrix.

\paragraph{High-dimensional instrumental variables.}
\citet{bai2010instrumental} assume $N$ noisy scalar instruments each dependent on $r$ latent factors, and use PCA to extract those factors as the effective instruments. Their bottleneck is on the instrument side (many noisy proxies $\to$ $r$ factors), and their core problem is a many-instruments finite-sample bias of order $N/T$ in 2SLS; population-level non-identifiability does not arise because the number of instruments grows. \citet{lin2026learning} consider the complementary regime: a low-dimensional instrument $\mathbf{Z} \in \mathbb{R}^k$ and a high-dimensional treatment $\mathbf{X} \in \mathbb{R}^m$ with $m \gg k$. They show that naive unsupervised compression of $\mathbf{X}$ creates an unblocked backdoor path $\mathbf{Z} \to \mathbf{X} \setminus \tilde{D} \to Y$, biasing the IV estimate, and propose recovering the instrument-relevant treatment directions via SVD of the reduced-form regression of $\mathbf{X}$ on $\mathbf{Z}$.

Our IV result in \Cref{sec:est} is related to both but differs in three key respects. First, we identify a \emph{population-level} non-identifiability: when $\dim(\mathbf{I}) = \dim(\mathbf{X}) = \dim(\mathbf{Y}) = d$ and $\operatorname{rank}(M_{IX}) = r < d$, the 2SLS problem has infinitely many solutions---a failure that arises regardless of sample size. Second, the oracle bottleneck 2SLS procedure uses the \emph{outcome-side} bottleneck $W^b_{XY}$ in the second stage, in addition to the instrument-side $W^b_{IX}$; these two subspaces are generically distinct, and conflating them (as a na\"{i}ve analogue of \citet{lin2026learning} would do) leaves a residual bias characterized by Remark~\ref{rem:proj_bias}. Third, all mechanism matrices in the SCBM---including $M_{XY}$---admit bottleneck factorizations, which is what makes our rank condition~\eqref{eq:rc} the checkable criterion for identification.

\paragraph{Connections to matching.}
Matching is a prominent technique for estimating the effect of an (often binary) treatment variable $T$ on an outcome $Y$ in the presence of confounding covariates $X_1,\dots, X_k$. Roughly speaking, the idea behind matching is to pair treated and untreated units that are as similar as possible in terms of their covariate values. \emph{Exact matching}, that is, matching directly in the covariate space $\mathcal{X}_1 \times \dots \times \mathcal{X}_k$, for instance by measuring the Euclidean distance of the covariate values of the units, is often unfeasible if the sample size is small compared to the number of covariates. In this case, matching is commonly done by measuring the similarity of units in a bottleneck space. For instance, \emph{propensity score matching} declares units similar that are close in propensity score $e(\x) = P(T=1 | \X=\x)$ \citep{rosenbaum1983central} which in our terminology is a $[0,1]$-valued bottleneck between the covariates and the treatment. \emph{Coarsened exact matching} \citep{blackwell2009cem} interpolates between exact matching and the propensity score by operating in a coarser bottleneck space.

\newpage
\section{Additional experimental results}
\subsection{Identifiability}\label{app:sec:id_viz}

\begin{figure*}[!h]
    \centering
    \includegraphics[width=\linewidth]{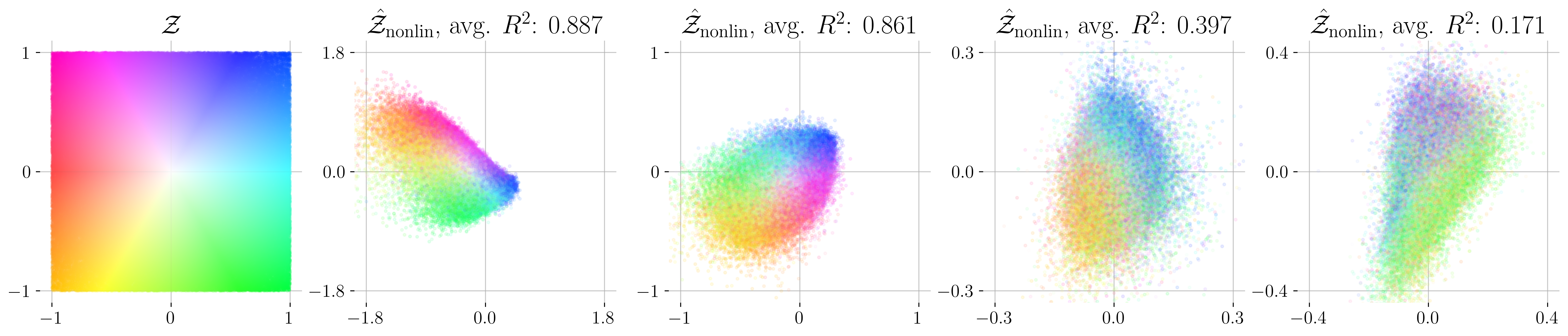}
    \caption{Visualizations of maps from ground truth bottleneck space $\mathcal{Z}$ (first plot) to different learned bottleneck spaces $\hat{\mathcal{Z}}_{\mathrm{nonlin}}$. The second and third plots show examples of successfully identified bottlenecks, as visualized by the smooth, structure preserving bijections from ground-truth to learned bottleneck spaces. We additionally show the identifiability metric, which in both cases is close to one. The third and fourth plots show examples where identifiability is not achieved (by not training to convergence). For such non-identified cases, the learned map is not a bijection and the local structure of the ground-truth bottleneck space is not preserved. The low $R^2$ metric in these failure cases further reflects non-identification of the ground-truth bottleneck.}
    \label{fig:id_viz_multi}
\end{figure*}

\subsection{Effect estimation}\label{app:ssec:adj_viz}
We repeat the mediator and backdoor adjustment experiments from \Cref{ssec:adj} with varying values of $d_{\X}$ and $d_{\Z}$. As a slight abuse of notation, we refer to the dimension of \emph{all} high-dimensional variables with $d_{\X}$, not just of the treatment $\X$. Results for the mediator case are plotted in \Cref{fig:med_sweep}, those for the backdoor case in \Cref{fig:bd_sweep}. Across settings, conditioning on the bottleneck variable provides a finite sample benefit over conditioning on the high-dimensional mediator or confounder. This advantage is most pronounced when the bottleneck dimension is smaller and diminishes as the ratio $d_{\X} / d_{\Z}$ approaches one.

\clearpage

\begin{figure}
    \centering
    \includegraphics[width=\linewidth]{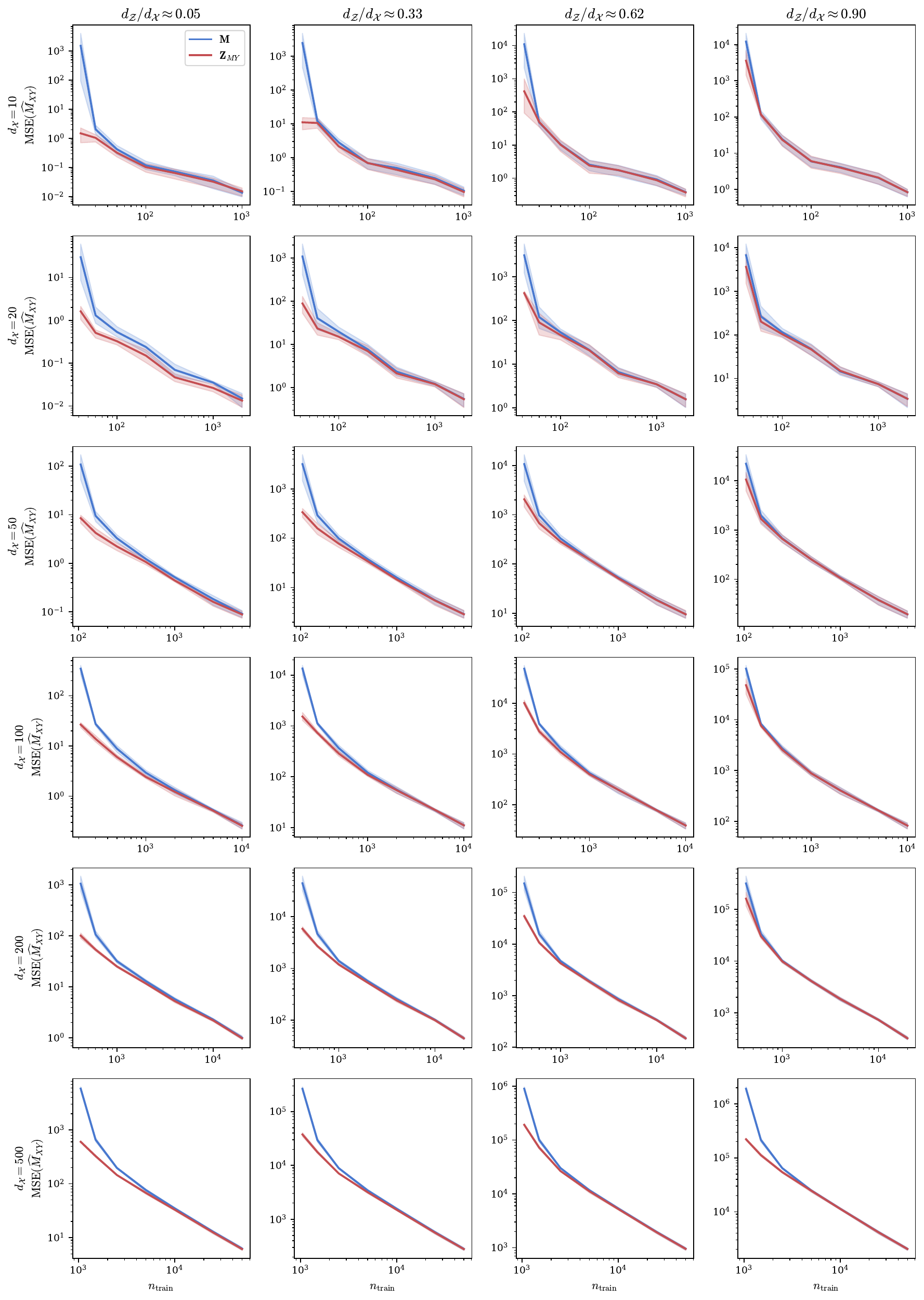}
    \caption{Results of the mediator adjustment experiment across various parameter configurations. Bottlenecks are preferable as conditioning variables over high-dimensional variables across settings. As $d_{\X} / d_{\Z} \to 1$, the finite sample benefits of bottlenecks diminish, with the notable exception of $d_{\X} =500$.}
    \label{fig:med_sweep}
\end{figure}

\clearpage

\begin{figure}
    \centering
    \includegraphics[width=\linewidth]{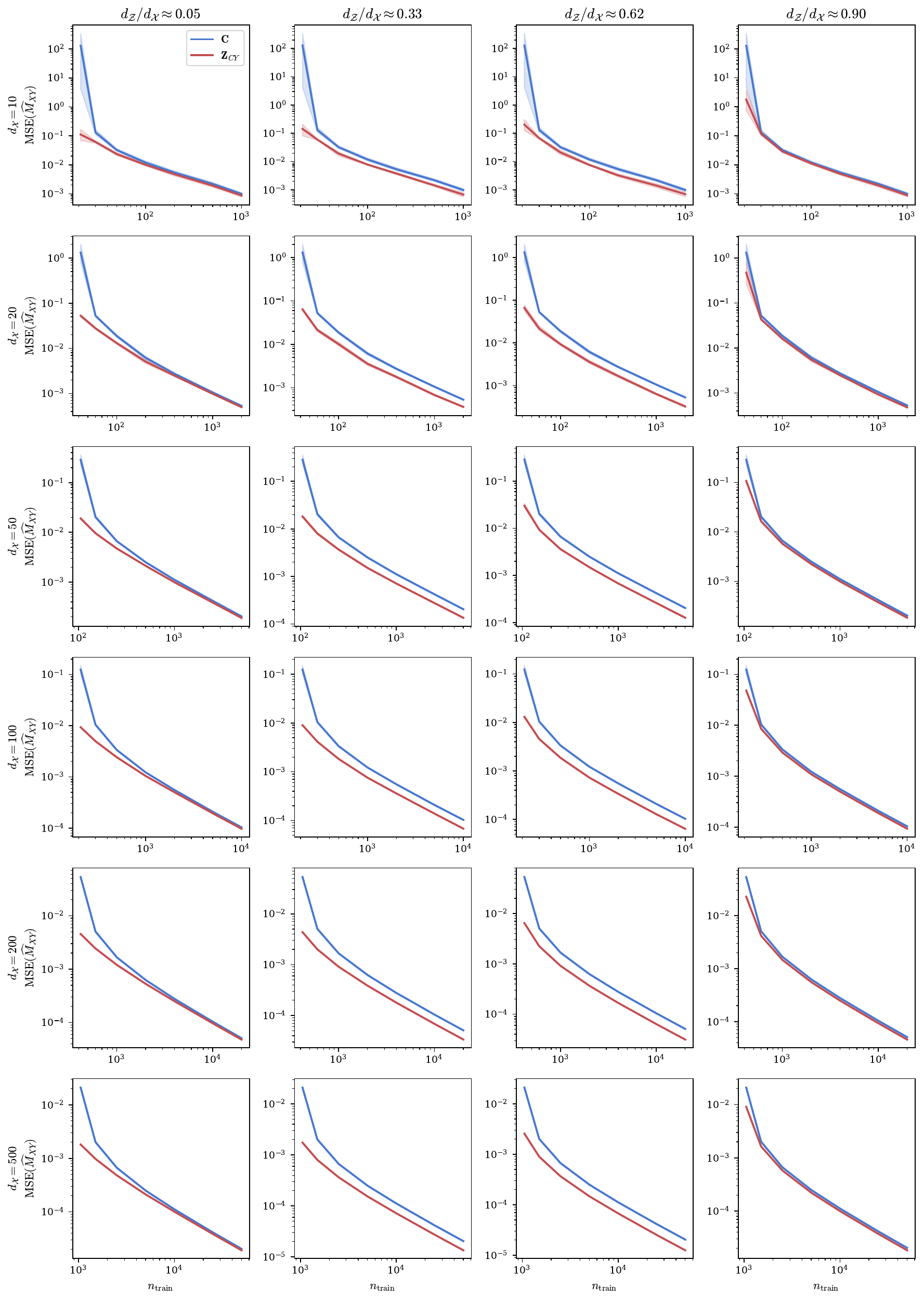}
    \caption{Results of the backdoor adjustment experiment across various parameter configurations. Bottlenecks are preferable as conditioning variables over high-dimensional variables across settings. As $d_{\X} / d_{\Z} \to 1$, the finite sample benefits of bottlenecks diminish. Interestingly, for intermediate values of $d_{\X} / d_{\Z}$ (middle two columns) the gap in estimation error of $M_{XY}$ when using the bottleneck vs. the high-dimensional confounder persists, even for large $n_{\textrm{train}}$.}
    \label{fig:bd_sweep}
\end{figure}

\clearpage

\subsection{Transfer environments for IV estimation}\label{app:transfer_iv}

Theorem~\ref{thm:oracle_2sls} assumes oracle access to $W^b_{IX}$ and $W^b_{XY}$. $W^b_{IX}$ can be estimated from the observational data, since the instrument is unconfounded. $W^b_{XY}$, by contrast, belongs to the confounded mechanism $\mathbf{X} \to \mathbf{Y}$. We therefore estimate it from a transfer environment that shares this mechanism; here a pilot RCT in which $\mathbf{X}$ is intervened on.

\paragraph{Setup.} We sample a linear SCBM as in Proposition~\ref{prop:2sls_nonid} with $d = 50$ and $r = 2$, and draw $n_{\mathrm{obs}} = 30000$ observational samples.
In contrast to the other experiments, we sample all weights from $\mathrm{Uniform}(-1.5, 1.5)$: strictly positive weights make the rank-$r$ mechanisms nearly rank one (median $\sigma_2/\sigma_1 = 0.04$ for $M_{XY}$, vs.\ $0.84$ with centered weights), which leaves the weak bottleneck direction poorly determined by the pilot and makes the estimated 2SLS unstable in $n_{\mathrm{pilot}}$.
For the pilot, we draw $n_{\mathrm{pilot}}$ samples from the same model under the intervention $\mathrm{do}(\mathbf{X} \sim \mathcal{N}(0, I_d))$, which removes the confounding through $\mathbf{U}$. We estimate $\widehat{W}^b_{IX}$ by rank-$r$ factorization of the first-stage regression of $\mathbf{X}$ on $\mathbf{I}$ in the observational data, and $\widehat{W}^b_{XY}$ by rank-$r$ factorization of the regression of $\mathbf{Y}$ on $\mathbf{X}$ in the pilot data. We compare four estimators of $M_{XY}$:
(i) \emph{oracle bottleneck 2SLS}, i.e.\ the procedure of Theorem~\ref{thm:oracle_2sls} with the true $W^b_{IX}, W^b_{XY}$;
(ii) \emph{estimated bottleneck 2SLS}, the same procedure with $\widehat{W}^b_{IX}, \widehat{W}^b_{XY}$;
(iii) \emph{ridge 2SLS} using the full $\mathbf{I}$ as instrument (penalty $\lambda = 10^{-3}$);
(iv) \emph{pilot-only OLS}, regressing $\mathbf{Y}$ on $\mathbf{X}$ in the pilot data alone.
We report the median MSE of $\widehat{M}_{XY}$ over 10 seeds in \cref{tab:transfer_iv}.

\begin{table}[h]
\centering
\caption{Median MSE of $\widehat{M}_{XY}$ over 10 seeds for varying pilot sizes ($n_{\mathrm{obs}} = 30000$, $d = 50$, $r = 2$).}
\label{tab:transfer_iv}
\begin{tabular}{rcccc}
\toprule
$n_{\mathrm{pilot}}$ & Oracle b-2SLS & Estimated b-2SLS & Ridge 2SLS & Pilot-only OLS \\
\midrule
100  & $1.2\times10^{-5}$ & $8.0$                & $1\times10^{3}$ & $1.1$ \\
500  & $1.2\times10^{-5}$ & $7.9\times10^{-2}$   & $1\times10^{3}$ & $1.3\times10^{-1}$ \\
2000 & $1.2\times10^{-5}$ & $4.9\times10^{-3}$   & $1\times10^{3}$ & $2.6\times10^{-2}$ \\
\bottomrule
\end{tabular}
\end{table}

\paragraph{Results.} Ridge 2SLS stays at the biased projection target of Remark~\ref{rem:proj_bias} regardless of sample size, confirming that the failure in Proposition~\ref{prop:2sls_nonid} is a population-level one. Estimated bottleneck 2SLS outperforms pilot-only OLS from $n_{\mathrm{pilot}} = 500$ onward and approaches the oracle as the pilot grows. The reason is that the pilot only needs to determine the $r$-dimensional subspace $\mathrm{col}(W^b_{XY})$, since all results are invariant under $W^b_{XY} \mapsto W^b_{XY}G$ (\cref{sec:est}). This amounts to $r(d-r)$ parameters, whereas pilot-only OLS must estimate all $d^2$ entries of $M_{XY}$. For very small pilots ($n_{\mathrm{pilot}} = 100$), however, the subspace estimate is too noisy, and pilot-only OLS performs better. Transfer environments are therefore most useful when a moderately sized pilot is available but too small to estimate the full effect matrix reliably.

\subsection{Approximately low-rank mechanisms}\label{app:ssec:misspec_lowrank}
Our efficiency results for bottleneck adjustment assume that mechanisms factor exactly through a rank-$r$ bottleneck. We study robustness to violations of this assumption in the backdoor setting of \Cref{sssec:bd} by replacing the confounder--outcome mechanism with
\begin{equation*}
    M_{CY} = W^b_{CY} W^f_{CY} + E, \qquad \rho := \frac{\|E\|_F^2}{\|W^b_{CY} W^f_{CY}\|_F^2},
\end{equation*}
where $E$ is a full-rank perturbation. Since $\mathrm{Cov}(\*C) = I$, $\rho$ is the variance contributed by $\*C E$ relative to that of $\*C\, W^b_{CY} W^f_{CY}$, and $\rho = 0$ recovers the original SCBM. We condition on the rank-$r$ bottleneck $\*Z_{CY} = \*C\, W^b_{CY}$ of the low-rank part, which violates outcome sufficiency (Lemma~\ref{lem:bd_validity}(a)) for any $\rho > 0$, and compare against conditioning on the full confounder $\*C$, which remains valid for all $\rho$. Details are given in Appendix~\ref{app:ssec:misspec_lowrank_details}.

\begin{figure}
    \centering
    \includegraphics[width=\linewidth]{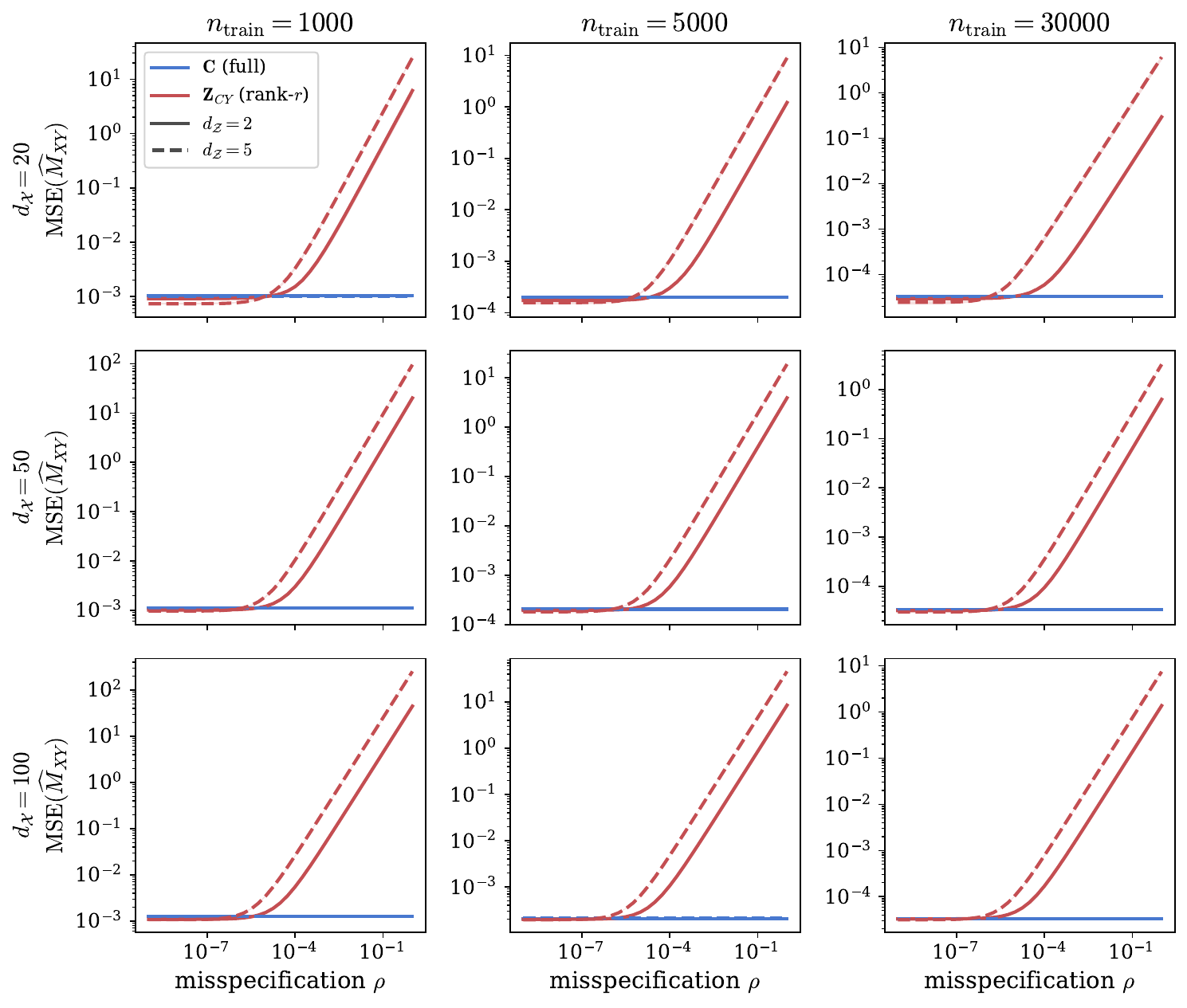}
    \caption{Backdoor adjustment under approximately low-rank $M_{CY}$. MSE of $\widehat{M}_{XY}$ vs.\ misspecification $\rho$ (log-log; $\rho > 0$ shown) for conditioning on the full confounder $\*C$ (\textcolor{scbmblue}{blue}) vs.\ the rank-$d_\Z$ bottleneck $\*Z_{CY}$ (\textcolor{scbmred}{red}), for $d_\Z = 2$ (solid) and $d_\Z = 5$ (dashed). Rows: $d_\X$; columns: $n_{\mathrm{train}}$. Lines show medians and bands interquartile ranges over 10 seeds. Full conditioning is unaffected by $\rho$, while the bottleneck loses its advantage at $\rho$ of order $10^{-7}$--$10^{-5}$.}
    \label{fig:misspec_grid}
\end{figure}

\begin{figure}
    \centering
    \includegraphics[width=\linewidth]{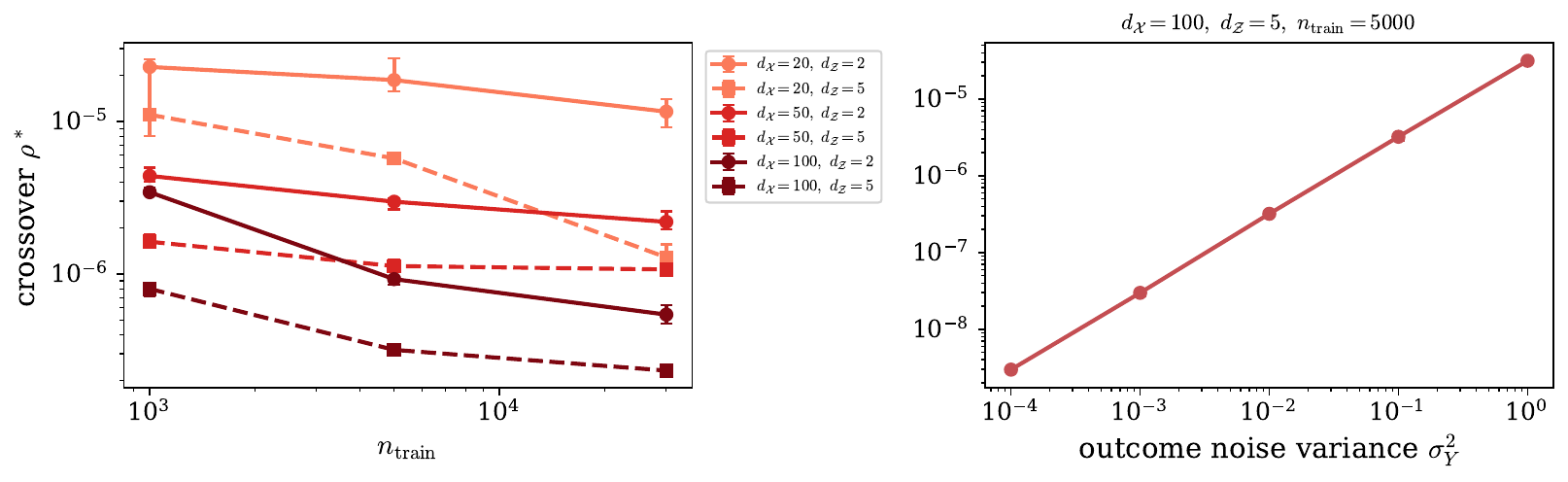}
    \caption{Crossover $\rho^\ast$ above which conditioning on the bottleneck is worse than conditioning on the full confounder. \textbf{Left:} $\rho^\ast$ vs.\ $n_{\mathrm{train}}$; $\rho^\ast$ decreases with sample size because the misspecification bias does not vanish. \textbf{Right:} $\rho^\ast$ vs.\ outcome noise variance $\sigma^2_Y$ ($d_\X = 100$, $d_\Z = 5$, $n_{\mathrm{train}} = 5000$); $\rho^\ast$ scales linearly with $\sigma^2_Y$. Markers show medians and error bars interquartile ranges of the per-seed crossover over 10 seeds.}
    \label{fig:misspec_crossover}
\end{figure}

\paragraph{Results.} The efficiency advantage of the bottleneck at $\rho = 0$ is modest (ratio of full to bottleneck MSE between $1.02$ and $1.36$) and is lost at very small perturbation levels (\Cref{fig:misspec_grid}). Across all configurations, the crossover $\rho^\ast$, i.e.\ the smallest $\rho$ at which the bottleneck MSE exceeds that of full conditioning, lies between $2.3 \cdot 10^{-7}$ and $2.3 \cdot 10^{-5}$. For the configuration of \Cref{fig:bottleneck_adjustment} ($d_\X = 100$, $d_\Z = 5$, $n_{\mathrm{train}} = 5000$), $\rho^\ast = 3.2 \cdot 10^{-7}$, and at $\rho = 0.02$ the MSE of $\widehat{M}_{XY}$ is $0.92$ with the bottleneck versus $2.1 \cdot 10^{-4}$ with full conditioning. The unmodelled component $\*C E$ affects the bottleneck estimator in two ways: it enters the residual of $\*Y$ as additional outcome noise, inflating the variance, and its part that is correlated with $\*X$ given $\*Z_{CY}$ biases $\widehat{M}_{XY}$. Unlike the variance, the bias does not vanish with $n_{\mathrm{train}}$, so $\rho^\ast$ decreases with sample size (by a factor of $1.5$ to $8.7$ from $n_{\mathrm{train}} = 10^3$ to $3 \cdot 10^4$; \Cref{fig:misspec_crossover}, left). The decrease is strongest where the bias dominates: for $d_\X = 20$, $d_\Z = 5$, it accounts for $89\%$ of the bottleneck MSE at $n_{\mathrm{train}} = 3\cdot 10^4$ and $\rho = 0.01$, compared with less than $1\%$ for $d_\X = 100$. Finally, $\rho^\ast$ scales linearly with the outcome noise variance, from $3.0 \cdot 10^{-9}$ at $\sigma^2_Y = 10^{-4}$ to $3.1 \cdot 10^{-5}$ at $\sigma^2_Y = 1$ (\Cref{fig:misspec_crossover}, right): what matters is the size of the perturbation relative to the noise, not relative to the low-rank part. The efficiency gains of bottleneck adjustment therefore hinge on the low-rank assumption holding essentially exactly. If it cannot be justified, one should condition on the full variables, which remain valid.

\subsection{Sample cost of estimating bottlenecks}\label{app:sample_cost}

The efficiency results of \cref{sec:est} assume oracle bottlenecks. To quantify the cost of estimating them, we repeat the mediator and backdoor experiments of \cref{fig:bottleneck_adjustment} ($d = 100$, $d_{\mathbf{Z}} = 5$), using 10 seeds under a fixed budget of $N$ samples and compare four strategies:
(i) \emph{full} conditioning on the high-dimensional mediator or confounder, using all $N$ samples;
(ii) a naive \emph{split}, estimating the bottleneck on $N/2$ samples and the effect on the remaining $N/2$;
(iii) \emph{cross-fitting}, which additionally swaps the roles of the two halves and averages, as is standard in double machine learning;
(iv) the \emph{oracle} bottleneck, using all $N$ samples.
\Cref{tab:sample_cost} reports mean $\pm$ std of the MSE of $\widehat{M}_{XY}$. Each seed draws a new mechanism, and the std mainly reflects variation across these draws. Since all strategies use the same data within a seed, we compare them seed by seed below.

\begin{table}[h]
\centering
\caption{MSE of $\widehat{M}_{XY}$ under a fixed sample budget $N$ (mean $\pm$ std over 10 seeds).}
\label{tab:sample_cost}
\small
\begin{tabular}{rcccc}
\toprule
$N$ & Full & Split & Cross-fit & Oracle \\
\midrule
\multicolumn{5}{l}{\emph{Mediator}} \\
$10^3$          & $3.3 \pm 0.64$ & $6.2 \pm 1.4$ & $3.1 \pm 0.52$ & $2.8 \pm 0.47$ \\
$5\cdot10^3$    & $0.52 \pm 0.046$ & $0.98 \pm 0.060$ & $0.50 \pm 0.053$ & $0.49 \pm 0.049$ \\
$3\cdot10^4$    & $0.087 \pm 0.014$ & $0.17 \pm 0.027$ & $0.085 \pm 0.013$ & $0.084 \pm 0.013$ \\
$10^5$          & $0.025 \pm 0.0046$ & $0.048 \pm 0.0075$ & $0.025 \pm 0.0039$ & $0.025 \pm 0.0038$ \\
$3\cdot10^5$    & $0.0090 \pm 0.0020$ & $0.017 \pm 0.0030$ & $0.0087 \pm 0.0020$ & $0.0087 \pm 0.0020$ \\
\midrule
\multicolumn{5}{l}{\emph{Backdoor} ($\times 10^{-5}$)} \\
$10^3$          & $130 \pm 1.3$ & $4700 \pm 4800$ & $1900 \pm 1600$ & $110 \pm 1.4$ \\
$5\cdot10^3$    & $21 \pm 0.42$ & $590 \pm 1100$ & $650 \pm 900$ & $19 \pm 0.41$ \\
$3\cdot10^4$    & $3.4 \pm 0.064$ & $9.9 \pm 8.0$ & $4.6 \pm 3.6$ & $3.2 \pm 0.065$ \\
$10^5$          & $1.0 \pm 0.0094$ & $1.9 \pm 0.021$ & $0.95 \pm 0.0093$ & $0.95 \pm 0.0095$ \\
$3\cdot10^5$    & $0.33 \pm 0.0043$ & $0.63 \pm 0.0067$ & $0.31 \pm 0.0042$ & $0.31 \pm 0.0042$ \\
\bottomrule
\end{tabular}
\end{table}

The naive split is consistently worst, as it halves the data available for effect estimation. In the mediator setting, cross-fitting is close to the oracle at all budgets and indistinguishable from it from $N = 10^4$ onward. Its advantage over full conditioning is small and within seed variability, mirroring the modest gap between oracle and full conditioning in this configuration. In the backdoor setting, estimating the rank-5 bottleneck from $N/2$ samples is unreliable at small budgets, driven by a few mechanism draws for which the bottleneck subspace is hard to estimate. Cross-fitting is therefore worse than full conditioning for $N \le 5\cdot10^3$ and comparable at $N = 3\cdot10^4$. From $N = 10^5$ onward, it is within $1\%$ of the oracle and improves on full conditioning for all 10 seeds.

Estimated bottlenecks thus retain the benefit of oracle bottlenecks once the budget suffices to estimate them, but can be costly below that. This cost is paid once per mechanism: since bottlenecks are properties of mechanisms, an estimated bottleneck can be reused across all queries involving the same mechanism, or estimated from a larger auxiliary dataset, whereas full conditioning incurs the high-dimensional cost on every query. Moreover, only the bottlenecks relevant to a given query need to be estimated, so this cost does not scale with the total number of edges in the graph.

\clearpage

\section{Experimental details}\label{app:exp_details}
\subsection{General data generation process}\label{app:ssec:datagen}
We generate data by first randomly sampling an SCBM $\mathfrak{C}$ and then drawing $n$ samples from the joint distribution induced by $\mathfrak{C}$. To sample an SCBM, we set the number of vertices $\mathcal{V}$, the internal dimension of the nodes $d_{\X}$ and the dimension of the bottleneck spaces $d_{\Z}$. With these fixed parameters, first the graph $\mathcal{G}$ is sampled from an Erdős–Rényi model \citep{erdos1959random} with edge probability $p=0.7$. Then, for each node $\X_j$ we sample the distribution of its respective noise term from a Markov Random Field whose joint distribution is a Gaussian and internal dynamics are described by the Langevin diffusion \citep{lauritzen2002chain} with dimension $d_{\X}$. For each edge in $\mathcal{G}$, we randomly sample both a bottleneck function $b_{ij}$, as well as an effect function $f_{ij}$. For linear models, we sample $b_{ij}$ by sampling a random matrix $W^b_{ij} \in [0,1]^{d_{\X} \times d_{\Z}}$ and $f_{ij}$ by sampling a random matrix $W^f_{ij} \in [0,1]^{d_{\Z} \times d_{\X}}$, both with $\text{rank} = d_{\Z}$. For nonlinear models, $b_{ij}$ and $f_{ij}$ are implemented by randomly initialized, multilayer perceptrons (MLPs). Details for each experimental setup are provided in Appendix~\ref{app:ssec:hyper}. See also Appendix~\ref{app:sec:mlps} for a detailed discussion on sampling nonlinear functions. Unless specified otherwise, by default we use the parameters $|\mathcal{V}| = 10$, $d_{\X} = 5$, $d_{\Z} = 2$ and $n = 30000$ for linear models and $n=50000$ for nonlinear models, respectively. We conduct experiments where we vary one of these parameters while keeping all others fixed.

\subsection{Detailed estimation procedure}\label{app:ssec:estim}


As per Definition~\ref{def.factored}, we consider the case where there is a separate bottleneck for each parent of an endogenous node. Graphically, this means there is a bottleneck space for each edge in the graph. Since each edge $\X_i \to \X_j$ can be decomposed into a bottleneck function $b_{ij}$, mapping to the bottleneck space $\cZ_{ij}$, and an effect function $f_{ij}$ mapping from $\cZ_{ij}$ to $\cX_j$, the targets of the estimation procedure are $b_{ij}$ and $f_{ij}$. We recover the joint map $m_{ij} := f_{ij} \circ b_{ij}: \cX_i \times \cH_j \to \cX_j$ by fitting an estimator from $\X_i$ to $\X_j$. By Lemma~\ref{lem.converse_direction}, from the estimated map $\hat{m}_{ij}$ we recover both $b_{ij}$ and $f_{ij}$ up to an invertible map $\psi_{ij}$, i.e., $\hat{b}_{ij} := \psi_{ij} \circ b_{ij}$ and $\hat{f}_{ij} := f_{ij} \circ \psi^{-1}_{ij}$.
How $\hat{b}_{ij}$ and $\hat{f}_{ij}$ are recovered from $\hat{m}_{ij}$ is described in the following. Code to reproduce all experimental results is available at \url{https://github.com/simonbing/StructuralCausalBottleneckModels}.
We implement our estimation pipeline using the JAX machine learning library \citep{jax2018github}.

\paragraph{Conditioning sets.}
Estimating $b_{ij}$ for edge $\X_i \to \X_j$ reduces to estimating the direct effect $\X_i \to \X_j$, so any conditioning set valid for that problem is valid here. The parent variables $\{\X_k\}_{k \in \pa(i)} \cup \{\X_\ell\}_{\ell \in \pa(j)}$ always constitute a valid, if high-dimensional, conditioning set. When bottleneck variables from prior estimation steps are available, they can be used instead. The results of \Cref{sssec:med,sssec:bd} govern which bottlenecks are valid substitutes: for a backdoor path through a confounder $\X_k$, either the treatment-side bottleneck $\Z_{ki}$ or the outcome-side bottleneck $\Z_{kj}$ is valid (Lemma~\ref{lem:bd_validity}); for a frontdoor path $\X_i \to \X_m \to \X_j$ through a mediator, only the outcome-side bottleneck $\Z_{mj}$ is valid---the treatment-side $\Z_{im}$ is a function of $\X_i$ and cannot block the indirect path (cf.\ Lemma~\ref{lem:med-validity}).

\paragraph{Linear vs. nonlinear estimators.}
Let $\X_i \in \mathbb R^{d_i}, \X_j \in \mathbb R^{d_j}$ and $\Z_{ij} \in \mathbb R^{d_{ij}}$. Fitting an estimator between $\X_i$ and $\X_j$ results in a map $\hat{m}_{ij}: \mathbb R^{d_i} \mapsto \mathbb R^{d_j}$ from which we must extract $\hat{b}_{ij}: \mathbb R^{d_i} \mapsto \mathbb R^{d_{ij}}$ and $\hat{f}_{ij}: \mathbb R^{d_{ij}} \mapsto \mathbb R^{d_j}$. In the linear case, fitting an estimator returns a weight matrix $\widehat{M}_{ij} \in \mathbb R^{d_i \times d_j}$, which by assumption has rank $d_{ij} \ll d_i, d_j$. Extracting $\hat{b}_{ij}$ and $\hat{f}_{ij}$ amounts to finding a matrix factorization $\widehat{M}_{ij} = \widehat{W}^b_{ij} \widehat{W}^f_{ij}$, where $\widehat{W}^b_{ij} \in \mathbb R^{d_i \times d_{ij}}$ and $\widehat{W}^f_{ij} \in \mathbb R^{d_{ij} \times d_j}$. We recover $\widehat{W}^b_{ij}$ by selecting $d_{ij}$ linearly independent columns from $\widehat{M}_{ij}$ and set $\widehat{W}^f_{ij} = {\widehat{W}^{b+}_{ij}} \widehat{M}_{ij}$, where $(\cdot)^+$ denotes the Moore-Penrose inverse \citep{moore1920reciprocal, bjerhammar1951application, penrose1955generalized}.
In the nonlinear case, we employ an encoder-decoder architecture: an encoder network $\hat{b}_{\theta}$ maps from $\X_i$ to $\Z_{ij}$ and a decoder network $\hat{f}_{\phi}$ maps from $\Z_{ij}$ to $\X_j$.

\begin{figure}[t]
    \centering
    \resizebox{.22\textwidth}{!}{
    \tikz[latent/.append style={minimum size=0.85cm},obs/.append style={minimum size=0.85cm},det/.append style={minimum size=1.15cm}, wrap/.append style={inner sep=2pt}, plate caption/.append style={below left=5pt and 0pt of #1.south east}, on grid]{
        \node[latent](x_1){$\X_1$};
        \node[latent, right=2.5cm of x_1](x_2){$\X_2$};
        \node[latent, below=2.5cm of x_2](x_3){$\X_3$};

        \edge[-{Stealth[length=2mm, width=1.5mm]}]{x_1}{x_2, x_3};
        \edge[-{Stealth[length=2mm, width=1.5mm]}]{x_2}{x_3};

    }
    }
    \caption{Three-node, fully connected graph.}
    \label{fig:toy_graph}
\end{figure}

\paragraph{Estimating all bottlenecks.}
Note that for a specific causal query, one need only estimate the bottlenecks on paths relevant to that query. The procedure below is needed only when the goal is to estimate \emph{all} bottlenecks in a given SCBM.

If our goal is to estimate \emph{all} bottlenecks in a given SCBM, the order of estimation matters: any bottleneck used as a conditioning variable must already have been estimated. We propose to loop over target nodes in causal order and, for each target, over its parents in \emph{reverse} causal order. To illustrate, consider the three-node graph in \Cref{fig:toy_graph} with bottlenecks $\Z_{12}, \Z_{13}, \Z_{23}$. We first estimate $\Z_{12}$ (no conditioning needed), then estimate $\Z_{23}$ conditioning on $\hat{\Z}_{12}$ to block the backdoor path $\X_2 \leftarrow \X_1 \rightarrow \X_3$ (valid by treatment-side sufficiency, Lemma~\ref{lem:bd_validity}(b)), and finally estimate $\Z_{13}$ conditioning on $\hat{\Z}_{23}$ to block the frontdoor path $\X_1 \rightarrow \X_2 \rightarrow \X_3$ (valid by outcome-side sufficiency, Lemma~\ref{lem:bd_validity}(a) applied to the mediator $\X_2$). At each step, all required conditioning bottlenecks have already been estimated.

This generalizes to arbitrary DAGs: backdoor paths through $\X_k \in \pa(\X_i)$ are blocked by $\Z_{ki}$, available because $\X_i$ was the target in a prior outer-loop iteration; frontdoor paths through $\X_\ell \in \pa(\X_j)$ are blocked by $\Z_{\ell j}$, available because $\X_\ell$ is later in causal order and was therefore processed first in the inner reverse loop. Pseudocode is given in \Cref{alg:estim}; the bottleneck estimate is $\hat{\Z}_{ij} := \hat{b}_{ij}(\X_i)$.

\begin{algorithm}[!h]
  \caption{Estimate bottleneck \& effect functions}
  \label{alg:estim}
  \begin{algorithmic}
  \State {\bfseries Input:} SCBM $\mathfrak{C} = \langle \cG = (\mathcal{V}, \mathcal{E}), \cX, \cZ, \cH,\boldsymbol{\eta}, \mathcal{B},\mathcal{F}, \X \rangle$, causal order $C[1..|\mathcal{V}|]$, assumed bottleneck dimension $d_{\hat{\Z}}$
  \State
  \For{$j$ {\bfseries in} $C$} \Comment{Outer loop over target nodes.}
  \State $target \gets \X_j$
  \If{$target$ is root node}
  \State pass \Comment{Root nodes have no bottlenecks.}
  \Else
  \For{$i$ {\bfseries in} reversed($C[:j]$)} \Comment{Inner loop over source nodes.}
  \State $source \gets \X_i$
  \State $\hat{\Z}_{cond} \gets GetCondSet(\mathfrak{C}, i, j)$ \Comment{According to procedure described in Appendix~\ref{app:ssec:estim}.}
  \State $\hat{m}_{ij} \gets fit(source, target, \hat{\Z}_{cond})$ \Comment{Using linear or nonlinear regressor.}
  \State $\hat{b}_{ij}, \hat{f}_{ij} \gets factorize(\hat{m}_{ij}, d_{\hat{\Z}})$
  \State Cache $\hat{b}_{ij}, \hat{f}_{ij}$
  \EndFor
  \EndIf
  \EndFor
  \end{algorithmic}
\end{algorithm}

\begin{proposition}[Correctness of Algorithm~\ref{alg:estim}]\label{prop:alg-correctness}
Let $\mathfrak{C}$ be a factored SCBM (Definition~\ref{def.factored}) with additive noise and almost-surely injective effect functions $\tilde{f}_{ij}$ for all $(i,j) \in \mathcal{E}$. If the assumed bottleneck dimensions equal the true ones, then at the population level Algorithm~\ref{alg:estim} returns, for each $(i,j) \in \mathcal{E}$,
\[
  \hat{b}_{ij} = \psi_{ij} \circ b_{ij}, \qquad \hat{f}_{ij} = f_{ij} \circ \psi_{ij}^{-1},
\]
for some invertible $\psi_{ij}$.
\end{proposition}

\begin{proof}
We proceed by induction on the iteration count of Algorithm~\ref{alg:estim}. Index iterations by pairs $(j, i)$ where $j$ is the outer-loop index and $i$ the inner-loop index; order them lexicographically with $j$ increasing and, within fixed $j$, $i$ decreasing (matching the reverse causal order of the inner loop).

\paragraph{Induction hypothesis.} All bottleneck estimates $\hat{b}_{k\ell}$ cached in iterations preceding $(j, i)$ satisfy $\hat{b}_{k\ell} = \psi_{k\ell} \circ b_{k\ell}$ for some invertible $\psi_{k\ell}$.

\paragraph{Induction step.} Consider estimating $b_{ij}$ in iteration $(j, i)$. The conditioning set $\hat{\Z}_{\mathrm{cond}}$ returned by \texttt{GetCondSet} consists of two types of previously estimated bottlenecks.
\begin{enumerate}
    \item[(i)] \emph{Backdoor conditioning}: for each $\X_k \in \pa(\X_i)$, \texttt{GetCondSet} returns $\hat{\Z}_{ki}$, estimated during outer-loop iteration $i$, which precedes $j$ in causal order so $(i, \cdot) < (j, i)$ in our ordering. By the induction hypothesis, $\hat{b}_{ki} = \psi_{ki} \circ b_{ki}$ for some invertible $\psi_{ki}$, so $\hat{\Z}_{ki}$ is an invertible reparametrization of $\Z_{ki}$. Since $\Z_{ki}$ satisfies treatment sufficiency ($\X_i \perp \X_k \mid \Z_{ki}$; Lemma~\ref{lem:bd_validity}(b)), and invertible reparametrization preserves conditional independence, $\hat{\Z}_{ki}$ is a valid conditioning variable for blocking the backdoor path through $\X_k$.
    \item[(ii)] \emph{Frontdoor conditioning}: for each $\X_\ell \in \pa(\X_j)$ with $\ell \neq i$ such that $\X_i \to \cdots \to \X_\ell$ in $\cG$, \texttt{GetCondSet} returns $\hat{\Z}_{\ell j}$. Since $\X_\ell$ is a descendant of $\X_i$, we have $i < \ell < j$ in causal order. The inner loop for target $j$ processes sources in reverse causal order, so $\X_\ell$ (with $\ell > i$) is processed in iteration $(j, \ell)$ before $(j, i)$. By the induction hypothesis, $\hat{b}_{\ell j} = \psi_{\ell j} \circ b_{\ell j}$ for some invertible $\psi_{\ell j}$, so $\hat{\Z}_{\ell j}$ is an invertible reparametrization of $\Z_{\ell j}$. Since $\Z_{\ell j}$ satisfies outcome sufficiency for the indirect path $\X_i \to \X_\ell \to \X_j$ (Lemma~\ref{lem:med-validity}), and invertible reparametrization preserves this property, $\hat{\Z}_{\ell j}$ is a valid conditioning variable for blocking the frontdoor path through $\X_\ell$.
\end{enumerate}
Together, (i) and (ii) imply that $\hat{\Z}_{\mathrm{cond}}$ is a valid adjustment set for the direct effect $\X_i \to \X_j$. At the population level, the regression of $\X_j$ on $(\X_i, \hat{\Z}_{\mathrm{cond}})$ therefore identifies the joint map $m_{ij} = f_{ij} \circ b_{ij}$. Since the effect functions are almost-surely injective and the SCBM has additive noise, Lemma~\ref{lem.converse_direction} applies: factorizing the estimated map yields $\hat{b}_{ij} = \psi_{ij} \circ b_{ij}$ and $\hat{f}_{ij} = f_{ij} \circ \psi_{ij}^{-1}$ for some invertible $\psi_{ij}$.

\paragraph{Base case.} The first iteration is $(j_{\min}, i_{\max})$, where $j_{\min}$ is the first non-root node in causal order and $i_{\max}$ is its highest-indexed predecessor. Since $j_{\min}$ is the first non-root, all nodes preceding it in causal order are roots, so $\pa(i_{\max}) = \emptyset$ (no backdoor conditioning needed) and there are no intermediary nodes on any path from $\X_{i_{\max}}$ to $\X_{j_{\min}}$ with higher causal index than $i_{\max}$ (no frontdoor conditioning needed). No prior estimates are required; Lemma~\ref{lem.converse_direction} applies directly.

By induction, all bottleneck estimates satisfy the identifiability guarantee.
\end{proof}

\subsection{Data generation details \& hyperparameters}\label{app:ssec:hyper}

\subsubsection{Sampling nonlinear bottleneck and effect functions}\label{app:ssec:id}

To sample nonlinear bottleneck and effect functions, we randomly initialize MLPs with two hidden layers with $d_{\textrm{hidden}} = d_{\X}$ and orthogonalize as described in Appendix~\ref{app:sec:mlps}. We use the swish nonlinearity. Hyperparameters of the estimator are described in \Cref{tab:id_hyper}.

\subsubsection{Estimation under misspecification}
The data generation process and hyperparameters for the misspecification experiments follow those of the identifiability experiments described above.

\begin{table}[!h]
    \centering
    \caption{Hyperparameters for identifiability and misspecification experiments.}
    \begin{sc}
    \begin{tabular}{lc}
    \toprule
        Hyperparameter & Value \\
        \midrule
        Encoder hidden layer sizes & $[256, 128, 64, 64, 32]$\\
        Decoder hidden layer sizes & $[32, 64, 64, 128, 256]$\\
        Epochs & 500\\
        Batch size & 1024\\
        Optimizer & AdamW \textnormal{\citep{loshchilovdecoupled}}\\
        Learning rate scheduler & Cosine decay w/ warmup\\
        Max. learning rate & $1\mathrm{e}{-5}$\\
        Min. learning rate & $1\mathrm{e}{-7}$\\
        Warmup steps & 1000\\
        Decay steps & 10000\\
        \bottomrule
    \end{tabular}
    \end{sc}
    \label{tab:id_hyper}
\end{table}

\subsubsection{Partitioning and intervention selection}\label{app:ssec:partitioning}

The input space partitioning experiment (\Cref{fig:iv-selection}) uses a three-node nonlinear SCBM with $\X_1, \X_2, \X_3 \in \mathbb{R}^2$ and graph $\X_1 \to \X_2$, $\X_1 \to \X_3$. The two ground-truth bottleneck functions are $b_{12}(\x) = \|\x\|_2$ (Euclidean norm, capturing radial distance) and $b_{13}(\x) = 2\cdot\mathbbm{1}[x^{(1)} > 0] + \mathbbm{1}[x^{(2)} > 0]$ (quadrant index, taking values in $\{0,1,2,3\}$); both map to a one-dimensional bottleneck space. Effect functions are linear: $f_{12}(\z) = \z W_{12}$ and $f_{13}(\z) = \z W_{13}$ with $W_{12}, W_{13} \in [0,1]^{1\times 2}$ sampled uniformly. The root $\X_1 \sim \mathcal{N}(0, I_2)$, and downstream noise terms $\eta_2, \eta_3 \sim \mathcal{N}(0, 0.09\, I_2)$. We train on $n = 10000$ samples.

Each bottleneck is estimated with an encoder–decoder network: an MLP encoder (hidden layer sizes $[64, 32]$, swish activations) mapping $\X_1 \in \mathbb{R}^2$ to a one-dimensional latent encoding, followed by a single linear decoder. Training uses AdamW \citep{loshchilovdecoupled} with learning rate $10^{-3}$; $b_{12}$ is trained for 2000 epochs and $b_{13}$ for 5000 epochs to account for the harder discrete target.

\subsubsection{Adjustment experiments}\label{app:ssec:adjustment}

The mediator and backdoor adjustment experiments (\Cref{fig:bottleneck_adjustment}) use a linear three-node SCBM with $d_\X = 100$ and $d_\Z = 5$. All weight matrices are sampled from $\text{Uniform}(0.5, 1.5)$ and are of full rank $d_\Z$. The root node $\X_0$ (the treatment in the mediator setting, the confounder in the backdoor setting) has noise covariance $I_{100}$; all other nodes have noise covariance $0.01\, I_{100}$.

\paragraph{Mediator adjustment.} The graph is $\X_0 \to \X_1 \to \X_2$, $\X_0 \to \X_2$. The target quantity is the direct effect $M_{02}$. We compare conditioning on the full mediator $\X_1 \in \mathbb{R}^{100}$ against the oracle outcome-side bottleneck $\Z_{12} = b_{12}(\X_1) \in \mathbb{R}^5$.

\paragraph{Backdoor adjustment.} The graph is $\X_0 \to \X_1 \to \X_2$, $\X_0 \to \X_2$, with $\X_0$ acting as a confounder. The target quantity is the direct effect $M_{12}$. We compare conditioning on the full confounder $\X_0 \in \mathbb{R}^{100}$ against the oracle outcome-side bottleneck $\Z_{02} = b_{02}(\X_0) \in \mathbb{R}^5$.

In both settings, oracle bottlenecks (i.e., the ground-truth linear projections) are used as conditioning variables, and the estimator is ordinary least squares. Results are averaged over 5 random seeds.

\subsubsection{Approximately Low-Rank Mechanisms}\label{app:ssec:misspec_lowrank_details}
We use the backdoor data-generating process of Appendix~\ref{app:ssec:adjustment}, with identical weight sampling, so that for $\rho = 0$ the mechanisms coincide with those of the backdoor experiment for the same seed. The perturbation $E$ has i.i.d.\ standard Gaussian entries and is rescaled such that $\|E\|_F^2 = \rho\, \|W^b_{CY} W^f_{CY}\|_F^2$. For each seed, the mechanisms, the perturbation direction and all noise draws are held fixed across $\rho$, so that curves over $\rho$ are paired. We sweep $\rho \in \{0\} \cup \{0.02, 0.05, 0.2\} \cup \{37$ log-spaced values in $[10^{-9}, 1]\}$, $d_\X \in \{20, 50, 100\}$, $d_\Z \in \{2, 5\}$ and $n_{\mathrm{train}} \in \{10^3, 5\cdot 10^3, 3\cdot 10^4\}$, with 10 seeds per configuration. The confounder has noise covariance $I$, and treatment and outcome have noise covariance $0.01\, I$. For the noise-level sweep, we fix $d_\X = 100$, $d_\Z = 5$, $n_{\mathrm{train}} = 5000$ and $\sigma^2_X = 0.01$, and vary $\sigma^2_Y \in \{10^{-4}, 10^{-3}, 10^{-2}, 10^{-1}, 1\}$. Both estimators are OLS with intercept of $\*Y$ on $(\*X, \*Z_{CY})$ and on $(\*X, \*C)$, respectively, using the oracle $W^b_{CY}$ of the low-rank part. The crossover $\rho^\ast$ is computed per seed by log-linear interpolation of the log-MSE ratio between adjacent grid points; we report its median and interquartile range over seeds.

\subsection{Compute resources}\label{app:ssec:comp}
For the identifiability experiments involving estimating nonlinear bottleneck functions we utilize an Nvidia A100 GPU; our experiments required approximately 200 GPU hours. All other experiments were performed on a local machine with an Apple M1 chip, requiring in the order of five to ten minutes, each.

\section{On randomly sampling nonlinear functions}\label{app:sec:mlps}
While visualizing ground-truth bottleneck spaces in nonlinear SCBMs, we stumbled across an unexpected phenomenon: randomly sampled ground-truth bottleneck spaces $\mathcal{Z}$ seemed to collapse to a one-dimensional subspace. Recalling that bottleneck functions map from observed to bottleneck spaces, i.e., $b_{{ij}}: \cX_i \rightarrow \cZ_{ij}$, we can see how this phenomenon has its explanation in the choice of ground-truth bottleneck functions. Following common practice, we sample these nonlinear bottleneck functions by randomly initializing MLPs.

A nonlinear function $m_{ij}: \mathbb R^{d_i} \rightarrow \mathbb R^{d_j}$ is implemented by an MLP with input dimension $d_i$ and output dimension $d_j$. Such an MLP is nothing more than the product of randomly sampled matrices---all layers with randomly sampled weights---composed with element-wise nonlinearities. Let us disregard the nonlinearities for now and take a closer look at products of random matrices, which we write as
\begin{align*}
    M = \prod^{h}_{i =1} W_i,
\end{align*}
where, for simplicity, we assume all factors $W_i$ to have compatible dimensions and full column rank. The number of factors $h$ can be understood as the number of hidden layers in the corresponding MLP. 

Let $r = \min_{i \in \{1..h\}} \textrm{rank}(W_i)$ denote the smallest rank of any factor in this product. We know that for the product it holds that $r = \textrm{rank}(M)$. However, empirically, we observe that the numerical rank of $M$ quickly approaches $1$. Quickly, in this context, means that this behaviour is observed for as few as three factors. This phenomenon also seems robust to the base distribution from which entries are chosen, as well as the size of the factor matrices. Understanding and characterizing this behaviour is active work in progress in an orthogonal research project.

While we cannot yet fully explain the origins of this collapse, we do propose an effective remedy in practice. When sampling the random matrices corresponding to individual matrices, we append an orthogonalization step. This is achieved by taking the $QR$-decomposition of the factor $W_i$ and taking $Q$ as the layer matrix. We thereby guarantee that we have the same span as the originally sampled $W_i$, but with an orthogonal basis. Empirically, this results in much more numerically stable full-rank products.

The technical difference between our proposed, new approach and the original one (without orthogonalization) is that we have proper control of the dimension of the space to which our MLP maps. Neither of the two is per se right or wrong, but their difference is worth being aware of. Rather, we believe our findings raise the question of how to ``properly" sample nonlinear functions in the first place. In our experiments, we see that the orthogonalization approach leads to significantly more complex functions (requiring more samples and resulting in slightly lower identifiability metrics), but how representative such synthetic functions are of practical problems remains an open question.

\section{Declaration of LLM usage}\label{app:llm}
This work used LLMs to assist with some of the proofs (primarily in \Cref{sec:est}), as well as for some of the code for the experiments. All proofs were verified for correctness by a human.


\newpage
\section*{NeurIPS Paper Checklist}

\begin{enumerate}

\item {\bf Claims}
    \item[] Question: Do the main claims made in the abstract and introduction accurately reflect the paper's contributions and scope?
    \item[] Answer: \answerYes{} 
    \item[] Justification: We state our main contributions at the end of \Cref{sec:intro} and support these claims in \Cref{sec:id} and \Cref{sec:est}.
    \item[] Guidelines:
    \begin{itemize}
        \item The answer \answerNA{} means that the abstract and introduction do not include the claims made in the paper.
        \item The abstract and/or introduction should clearly state the claims made, including the contributions made in the paper and important assumptions and limitations. A \answerNo{} or \answerNA{} answer to this question will not be perceived well by the reviewers. 
        \item The claims made should match theoretical and experimental results, and reflect how much the results can be expected to generalize to other settings. 
        \item It is fine to include aspirational goals as motivation as long as it is clear that these goals are not attained by the paper. 
    \end{itemize}

\item {\bf Limitations}
    \item[] Question: Does the paper discuss the limitations of the work performed by the authors?
    \item[] Answer: \answerYes{} 
    \item[] Justification: We discuss limitations of our work in \Cref{sec:discussion}.
    \item[] Guidelines:
    \begin{itemize}
        \item The answer \answerNA{} means that the paper has no limitation while the answer \answerNo{} means that the paper has limitations, but those are not discussed in the paper. 
        \item The authors are encouraged to create a separate ``Limitations'' section in their paper.
        \item The paper should point out any strong assumptions and how robust the results are to violations of these assumptions (e.g., independence assumptions, noiseless settings, model well-specification, asymptotic approximations only holding locally). The authors should reflect on how these assumptions might be violated in practice and what the implications would be.
        \item The authors should reflect on the scope of the claims made, e.g., if the approach was only tested on a few datasets or with a few runs. In general, empirical results often depend on implicit assumptions, which should be articulated.
        \item The authors should reflect on the factors that influence the performance of the approach. For example, a facial recognition algorithm may perform poorly when image resolution is low or images are taken in low lighting. Or a speech-to-text system might not be used reliably to provide closed captions for online lectures because it fails to handle technical jargon.
        \item The authors should discuss the computational efficiency of the proposed algorithms and how they scale with dataset size.
        \item If applicable, the authors should discuss possible limitations of their approach to address problems of privacy and fairness.
        \item While the authors might fear that complete honesty about limitations might be used by reviewers as grounds for rejection, a worse outcome might be that reviewers discover limitations that aren't acknowledged in the paper. The authors should use their best judgment and recognize that individual actions in favor of transparency play an important role in developing norms that preserve the integrity of the community. Reviewers will be specifically instructed to not penalize honesty concerning limitations.
    \end{itemize}

\item {\bf Theory assumptions and proofs}
    \item[] Question: For each theoretical result, does the paper provide the full set of assumptions and a complete (and correct) proof?
    \item[] Answer: \answerYes{} 
    \item[] Justification: All theoretical results and proofs are presented in Appendix~\ref{app:proofs}.
    \item[] Guidelines:
    \begin{itemize}
        \item The answer \answerNA{} means that the paper does not include theoretical results. 
        \item All the theorems, formulas, and proofs in the paper should be numbered and cross-referenced.
        \item All assumptions should be clearly stated or referenced in the statement of any theorems.
        \item The proofs can either appear in the main paper or the supplemental material, but if they appear in the supplemental material, the authors are encouraged to provide a short proof sketch to provide intuition. 
        \item Inversely, any informal proof provided in the core of the paper should be complemented by formal proofs provided in appendix or supplemental material.
        \item Theorems and Lemmas that the proof relies upon should be properly referenced. 
    \end{itemize}

    \item {\bf Experimental result reproducibility}
    \item[] Question: Does the paper fully disclose all the information needed to reproduce the main experimental results of the paper to the extent that it affects the main claims and/or conclusions of the paper (regardless of whether the code and data are provided or not)?
    \item[] Answer: \answerYes{} 
    \item[] Justification: We provide all necessary details to reproduce our experiments in Appendix~\ref{app:exp_details}. All code to implement our experiments, along with instructions to reproduce the results we show is available at \url{https://github.com/simonbing/StructuralCausalBottleneckModels}.
    \item[] Guidelines:
    \begin{itemize}
        \item The answer \answerNA{} means that the paper does not include experiments.
        \item If the paper includes experiments, a \answerNo{} answer to this question will not be perceived well by the reviewers: Making the paper reproducible is important, regardless of whether the code and data are provided or not.
        \item If the contribution is a dataset and\slash or model, the authors should describe the steps taken to make their results reproducible or verifiable. 
        \item Depending on the contribution, reproducibility can be accomplished in various ways. For example, if the contribution is a novel architecture, describing the architecture fully might suffice, or if the contribution is a specific model and empirical evaluation, it may be necessary to either make it possible for others to replicate the model with the same dataset, or provide access to the model. In general. releasing code and data is often one good way to accomplish this, but reproducibility can also be provided via detailed instructions for how to replicate the results, access to a hosted model (e.g., in the case of a large language model), releasing of a model checkpoint, or other means that are appropriate to the research performed.
        \item While NeurIPS does not require releasing code, the conference does require all submissions to provide some reasonable avenue for reproducibility, which may depend on the nature of the contribution. For example
        \begin{enumerate}
            \item If the contribution is primarily a new algorithm, the paper should make it clear how to reproduce that algorithm.
            \item If the contribution is primarily a new model architecture, the paper should describe the architecture clearly and fully.
            \item If the contribution is a new model (e.g., a large language model), then there should either be a way to access this model for reproducing the results or a way to reproduce the model (e.g., with an open-source dataset or instructions for how to construct the dataset).
            \item We recognize that reproducibility may be tricky in some cases, in which case authors are welcome to describe the particular way they provide for reproducibility. In the case of closed-source models, it may be that access to the model is limited in some way (e.g., to registered users), but it should be possible for other researchers to have some path to reproducing or verifying the results.
        \end{enumerate}
    \end{itemize}

\item {\bf Open access to data and code}
    \item[] Question: Does the paper provide open access to the data and code, with sufficient instructions to faithfully reproduce the main experimental results, as described in supplemental material?
    \item[] Answer: \answerYes{} 
    \item[] Justification: Our code together with instructions on how to reproduce our results is available at \url{https://github.com/simonbing/StructuralCausalBottleneckModels}.
    \item[] Guidelines:
    \begin{itemize}
        \item The answer \answerNA{} means that paper does not include experiments requiring code.
        \item Please see the NeurIPS code and data submission guidelines (\url{https://neurips.cc/public/guides/CodeSubmissionPolicy}) for more details.
        \item While we encourage the release of code and data, we understand that this might not be possible, so \answerNo{} is an acceptable answer. Papers cannot be rejected simply for not including code, unless this is central to the contribution (e.g., for a new open-source benchmark).
        \item The instructions should contain the exact command and environment needed to run to reproduce the results. See the NeurIPS code and data submission guidelines (\url{https://neurips.cc/public/guides/CodeSubmissionPolicy}) for more details.
        \item The authors should provide instructions on data access and preparation, including how to access the raw data, preprocessed data, intermediate data, and generated data, etc.
        \item The authors should provide scripts to reproduce all experimental results for the new proposed method and baselines. If only a subset of experiments are reproducible, they should state which ones are omitted from the script and why.
        \item At submission time, to preserve anonymity, the authors should release anonymized versions (if applicable).
        \item Providing as much information as possible in supplemental material (appended to the paper) is recommended, but including URLs to data and code is permitted.
    \end{itemize}

\item {\bf Experimental setting/details}
    \item[] Question: Does the paper specify all the training and test details (e.g., data splits, hyperparameters, how they were chosen, type of optimizer) necessary to understand the results?
    \item[] Answer: \answerYes{} 
    \item[] Justification: We provide all experimental details in Appendix~\ref{app:exp_details}.
    \item[] Guidelines:
    \begin{itemize}
        \item The answer \answerNA{} means that the paper does not include experiments.
        \item The experimental setting should be presented in the core of the paper to a level of detail that is necessary to appreciate the results and make sense of them.
        \item The full details can be provided either with the code, in appendix, or as supplemental material.
    \end{itemize}

\item {\bf Experiment statistical significance}
    \item[] Question: Does the paper report error bars suitably and correctly defined or other appropriate information about the statistical significance of the experiments?
    \item[] Answer: \answerYes{}{} 
    \item[] Justification: Figure 2 and all ables show mean ± std; Figures 9 and 10 show median and IQR over 10 seeds; all other Figures show $95\%$ confidence interval bars, calculated over five random seeds via bootstrapping.
    \item[] Guidelines:
    \begin{itemize}
        \item The answer \answerNA{} means that the paper does not include experiments.
        \item The authors should answer \answerYes{} if the results are accompanied by error bars, confidence intervals, or statistical significance tests, at least for the experiments that support the main claims of the paper.
        \item The factors of variability that the error bars are capturing should be clearly stated (for example, train/test split, initialization, random drawing of some parameter, or overall run with given experimental conditions).
        \item The method for calculating the error bars should be explained (closed form formula, call to a library function, bootstrap, etc.)
        \item The assumptions made should be given (e.g., Normally distributed errors).
        \item It should be clear whether the error bar is the standard deviation or the standard error of the mean.
        \item It is OK to report 1-sigma error bars, but one should state it. The authors should preferably report a 2-sigma error bar than state that they have a 96\% CI, if the hypothesis of Normality of errors is not verified.
        \item For asymmetric distributions, the authors should be careful not to show in tables or figures symmetric error bars that would yield results that are out of range (e.g., negative error rates).
        \item If error bars are reported in tables or plots, the authors should explain in the text how they were calculated and reference the corresponding figures or tables in the text.
    \end{itemize}

\item {\bf Experiments compute resources}
    \item[] Question: For each experiment, does the paper provide sufficient information on the computer resources (type of compute workers, memory, time of execution) needed to reproduce the experiments?
    \item[] Answer: \answerYes{} 
    \item[] Justification: We provide details on compute resources in Appendix~\ref{app:ssec:comp}.
    \item[] Guidelines:
    \begin{itemize}
        \item The answer \answerNA{} means that the paper does not include experiments.
        \item The paper should indicate the type of compute workers CPU or GPU, internal cluster, or cloud provider, including relevant memory and storage.
        \item The paper should provide the amount of compute required for each of the individual experimental runs as well as estimate the total compute. 
        \item The paper should disclose whether the full research project required more compute than the experiments reported in the paper (e.g., preliminary or failed experiments that didn't make it into the paper). 
    \end{itemize}
    
\item {\bf Code of ethics}
    \item[] Question: Does the research conducted in the paper conform, in every respect, with the NeurIPS Code of Ethics \url{https://neurips.cc/public/EthicsGuidelines}?
    \item[] Answer: \answerYes{} 
    \item[] Justification: The research conducted in this work conforms with the NeurIPS Code of Ethics.
    \item[] Guidelines:
    \begin{itemize}
        \item The answer \answerNA{} means that the authors have not reviewed the NeurIPS Code of Ethics.
        \item If the authors answer \answerNo, they should explain the special circumstances that require a deviation from the Code of Ethics.
        \item The authors should make sure to preserve anonymity (e.g., if there is a special consideration due to laws or regulations in their jurisdiction).
    \end{itemize}

\item {\bf Broader impacts}
    \item[] Question: Does the paper discuss both potential positive societal impacts and negative societal impacts of the work performed?
    \item[] Answer: \answerNA{} 
    \item[] Justification: This work introduces and discusses a novel modeling framework, which amounts to basic research. While we hope this work can contribute to the development of practical methods with a positive societal impact, the results we discuss in this work are not expected to have a direct societal impact.
    \item[] Guidelines:
    \begin{itemize}
        \item The answer \answerNA{} means that there is no societal impact of the work performed.
        \item If the authors answer \answerNA{} or \answerNo, they should explain why their work has no societal impact or why the paper does not address societal impact.
        \item Examples of negative societal impacts include potential malicious or unintended uses (e.g., disinformation, generating fake profiles, surveillance), fairness considerations (e.g., deployment of technologies that could make decisions that unfairly impact specific groups), privacy considerations, and security considerations.
        \item The conference expects that many papers will be foundational research and not tied to particular applications, let alone deployments. However, if there is a direct path to any negative applications, the authors should point it out. For example, it is legitimate to point out that an improvement in the quality of generative models could be used to generate Deepfakes for disinformation. On the other hand, it is not needed to point out that a generic algorithm for optimizing neural networks could enable people to train models that generate Deepfakes faster.
        \item The authors should consider possible harms that could arise when the technology is being used as intended and functioning correctly, harms that could arise when the technology is being used as intended but gives incorrect results, and harms following from (intentional or unintentional) misuse of the technology.
        \item If there are negative societal impacts, the authors could also discuss possible mitigation strategies (e.g., gated release of models, providing defenses in addition to attacks, mechanisms for monitoring misuse, mechanisms to monitor how a system learns from feedback over time, improving the efficiency and accessibility of ML).
    \end{itemize}
    
\item {\bf Safeguards}
    \item[] Question: Does the paper describe safeguards that have been put in place for responsible release of data or models that have a high risk for misuse (e.g., pre-trained language models, image generators, or scraped datasets)?
    \item[] Answer: \answerNA{} 
    \item[] Justification: We do not release any code or assets that have a high risk for misuse.
    \item[] Guidelines:
    \begin{itemize}
        \item The answer \answerNA{} means that the paper poses no such risks.
        \item Released models that have a high risk for misuse or dual-use should be released with necessary safeguards to allow for controlled use of the model, for example by requiring that users adhere to usage guidelines or restrictions to access the model or implementing safety filters. 
        \item Datasets that have been scraped from the Internet could pose safety risks. The authors should describe how they avoided releasing unsafe images.
        \item We recognize that providing effective safeguards is challenging, and many papers do not require this, but we encourage authors to take this into account and make a best faith effort.
    \end{itemize}

\item {\bf Licenses for existing assets}
    \item[] Question: Are the creators or original owners of assets (e.g., code, data, models), used in the paper, properly credited and are the license and terms of use explicitly mentioned and properly respected?
    \item[] Answer: \answerYes{} 
    \item[] Justification: We cite the authors of all used software packages.
    \item[] Guidelines:
    \begin{itemize}
        \item The answer \answerNA{} means that the paper does not use existing assets.
        \item The authors should cite the original paper that produced the code package or dataset.
        \item The authors should state which version of the asset is used and, if possible, include a URL.
        \item The name of the license (e.g., CC-BY 4.0) should be included for each asset.
        \item For scraped data from a particular source (e.g., website), the copyright and terms of service of that source should be provided.
        \item If assets are released, the license, copyright information, and terms of use in the package should be provided. For popular datasets, \url{paperswithcode.com/datasets} has curated licenses for some datasets. Their licensing guide can help determine the license of a dataset.
        \item For existing datasets that are re-packaged, both the original license and the license of the derived asset (if it has changed) should be provided.
        \item If this information is not available online, the authors are encouraged to reach out to the asset's creators.
    \end{itemize}

\item {\bf New assets}
    \item[] Question: Are new assets introduced in the paper well documented and is the documentation provided alongside the assets?
    \item[] Answer: \answerYes{} 
    \item[] Justification: Our code and its documentation is available at \url{https://github.com/simonbing/StructuralCausalBottleneckModels}.
    \item[] Guidelines:
    \begin{itemize}
        \item The answer \answerNA{} means that the paper does not release new assets.
        \item Researchers should communicate the details of the dataset\slash code\slash model as part of their submissions via structured templates. This includes details about training, license, limitations, etc. 
        \item The paper should discuss whether and how consent was obtained from people whose asset is used.
        \item At submission time, remember to anonymize your assets (if applicable). You can either create an anonymized URL or include an anonymized zip file.
    \end{itemize}

\item {\bf Crowdsourcing and research with human subjects}
    \item[] Question: For crowdsourcing experiments and research with human subjects, does the paper include the full text of instructions given to participants and screenshots, if applicable, as well as details about compensation (if any)? 
    \item[] Answer: \answerNA{} 
    \item[] Justification: This work does not involve crowdsourcing nor research with human subjects.
    \item[] Guidelines:
    \begin{itemize}
        \item The answer \answerNA{} means that the paper does not involve crowdsourcing nor research with human subjects.
        \item Including this information in the supplemental material is fine, but if the main contribution of the paper involves human subjects, then as much detail as possible should be included in the main paper. 
        \item According to the NeurIPS Code of Ethics, workers involved in data collection, curation, or other labor should be paid at least the minimum wage in the country of the data collector. 
    \end{itemize}

\item {\bf Institutional review board (IRB) approvals or equivalent for research with human subjects}
    \item[] Question: Does the paper describe potential risks incurred by study participants, whether such risks were disclosed to the subjects, and whether Institutional Review Board (IRB) approvals (or an equivalent approval/review based on the requirements of your country or institution) were obtained?
    \item[] Answer: \answerNA{} 
    \item[] Justification: This work does not involve crowdsourcing nor research with human subjects.
    \item[] Guidelines:
    \begin{itemize}
        \item The answer \answerNA{} means that the paper does not involve crowdsourcing nor research with human subjects.
        \item Depending on the country in which research is conducted, IRB approval (or equivalent) may be required for any human subjects research. If you obtained IRB approval, you should clearly state this in the paper. 
        \item We recognize that the procedures for this may vary significantly between institutions and locations, and we expect authors to adhere to the NeurIPS Code of Ethics and the guidelines for their institution. 
        \item For initial submissions, do not include any information that would break anonymity (if applicable), such as the institution conducting the review.
    \end{itemize}

\item {\bf Declaration of LLM usage}
    \item[] Question: Does the paper describe the usage of LLMs if it is an important, original, or non-standard component of the core methods in this research? Note that if the LLM is used only for writing, editing, or formatting purposes and does \emph{not} impact the core methodology, scientific rigor, or originality of the research, declaration is not required.
    \item[] Answer: \answerYes{} 
    \item[] Justification: We declare LLM usage in Appendix~\ref{app:llm}.
    \item[] Guidelines:
    \begin{itemize}
        \item The answer \answerNA{} means that the core method development in this research does not involve LLMs as any important, original, or non-standard components.
        \item Please refer to our LLM policy in the NeurIPS handbook for what should or should not be described.
    \end{itemize}

\end{enumerate}

\end{document}